%% file: main.tex
\documentclass[11pt]{article}
\usepackage[left=1in,top=1in,right=1in,bottom=1in,letterpaper]{geometry}

\usepackage[utf8]{inputenc} %
\usepackage[T1]{fontenc}    %
\usepackage{hyperref}       %
\usepackage{url}            %
\usepackage{booktabs}       %
\usepackage{amsfonts}       %
\usepackage{nicefrac}       %
\usepackage{microtype}      %
\usepackage{xcolor}         %
\usepackage[most]{tcolorbox}
\usepackage{enumitem}
\usepackage{natbib}
\usepackage{subfigure}
\usepackage{placeins}
\usepackage{tikz}
\usetikzlibrary{matrix,decorations.pathreplacing}
\usetikzlibrary{positioning, arrows.meta, shapes.geometric, fit}
\usepackage{float}
\usepackage[most]{tcolorbox}
\usetikzlibrary{shadows}

\newcounter{exa}

\usepackage{amsmath}
\usepackage{amssymb}
\usepackage{mathtools}
\usepackage{amsthm}

\usepackage{subcaption}
\usepackage{keyval}

\makeatletter
\let\oldKV@Gin@width\KV@Gin@width

\define@key{Gin}{width}{%
  \begingroup
    \setlength{\@tempdima}{#1}%
    \setlength{\@tempdima}{0.5\@tempdima}%
    \expandafter
  \endgroup
  \expandafter\oldKV@Gin@width\expandafter{\the\@tempdima}%
}

\makeatother

\usepackage[capitalize,noabbrev]{cleveref}

\theoremstyle{plain}
\newtheorem{theorem}{Theorem}[section]
\newtheorem{proposition}[theorem]{Proposition}
\newtheorem{lemma}[theorem]{Lemma}

\theoremstyle{definition}

\newtheorem{assumption}[theorem]{Assumption}
\theoremstyle{remark}

\input{notations}
\title{Breaking the Reversal Curse in Autoregressive Language Models via Identity Bridge}
\author{
Xutao Ma\thanks{Equal contributions.}\qquad
Yixiao Huang\footnotemark[1]\qquad
Hanlin Zhu\qquad
Somayeh Sojoudi
\\
\\
UC Berkeley
\\
\\
\texttt{\{xutao\_ma,yixiaoh,hanlinzhu,sojoudi\}@berkeley.edu}
}
\date{}

\begin{document}

\maketitle

\begin{abstract}
Autoregressive large language models (LLMs) have achieved remarkable success in many complex tasks, yet they can still fail in very simple logical reasoning such as the ``reversal curse'' --- when trained on forward knowledge data of the form ``$A \rightarrow B$'' (e.g., \emph{Alice's husband is Bob}), the model is unable to deduce the reversal knowledge ``$B \leftarrow A$'' (e.g., \emph{Bob's wife is Alice}) during test. Extensive prior research suggests that this failure is an inherent, fundamental limit of autoregressive causal LLMs, indicating that these models tend to memorize factual-level knowledge rather than capture higher-level rules.
In this paper, we challenge this view by showing that this seemingly fundamental limit can be mitigated by slightly tweaking the training data with a simple regularization data recipe called the Identity Bridge of the form ``$A \to A$'' (e.g., \emph{The name of Alice is Alice}). 
Theoretically, we prove that under this recipe, even a one-layer transformer can break the reversal curse by analyzing the implicit bias of gradient descent. Empirically, we show that a 1B pretrained language model finetuned with the proposed data recipe achieves a 50\% success rate on reversal tasks, in stark contrast to a near-zero success rate when trained solely on forward-knowledge data. Our work provides a novel theoretical foundation for the reversal curse and offers a principled, low-cost path to encouraging LLMs to learn higher-level rules from data.
\end{abstract}

\input{contents/intro}

\input{contents/prelim}

\input{contents/theory}
\input{contents/experiments}
\input{contents/related_work}
\input{contents/conclusion}
\input{contents/ack}

\bibliography{references}
\bibliographystyle{plainnat}

\newpage
\appendix
\onecolumn

\include{appendix/appendix}

\end{document}

%% file: notations.tex
\newcommand{\cA}{\mathcal{A}}
\newcommand{\cB}{\mathcal{B}}

\newcommand{\cD}{\mathcal{D}}
\newcommand{\cV}{\mathcal{V}}
\newcommand{\cL}{\mathcal{L}}
\newcommand{\cR}{\mathcal{R}}
\newcommand{\cS}{\mathcal{S}}
\newcommand{\cI}{\mathcal{I}}
\newcommand{\cF}{\mathcal{F}}

\newcommand{\TF}{\text{TF}}

\newcommand{\bmg}{\boldsymbol{\beta - \alpha}}

\newcommand{\bz}{\boldsymbol{z}}
\newcommand{\bone}{\boldsymbol{1}}
\newcommand{\bI}{\boldsymbol{I}}
\newcommand{\be}{\boldsymbol{e}}
\newcommand{\bE}{\boldsymbol{E}}
\newcommand{\bx}{\boldsymbol{x}}
\newcommand{\bW}{\boldsymbol{W}}
\newcommand{\bWO}{\boldsymbol{W}_\text{O}}
\newcommand{\bWV}{\boldsymbol{W}_\text{V}}
\newcommand{\bWKQ}{\boldsymbol{W}_\text{KQ}}
\newcommand{\bWOV}{\boldsymbol{W}_\text{OV}}
\newcommand{\bX}{\boldsymbol{X}}

\newcommand{\btheta}{\boldsymbol{\theta}}
\newcommand{\bzero}{\boldsymbol{0}}

\newcommand{\R}{\mathbb{R}}
\newcommand{\E}{\mathbb{E}}

\newif
\ifdraft

\ifdraft
    
    \newcommand{\shaw}[1]{\textcolor{violet}{#1}}
    \newcommand{\shawp}[1]{\textcolor{violet}{[#1]}}

    \newcommand{\hanlin}[1]{\textcolor{blue}{[Hanlin: #1]}}
    \newcommand{\xutao}[1]{\textcolor{purple}{[Xutao: #1]}}

\else     
    
    \newcommand{\shaw}[1]{}
    
    \newcommand{\shawp}[1]{}

    \newcommand{\hanlin}[1]{}
    \newcommand{\xutao}[1]{}

\fi

\newcommand{\sublabel}[1]{%
  \par
 \small(#1)%
}

%% file: contents/intro.tex
\section{Introduction}
Autoregressive large language models (LLMs) have demonstrated great capability in solving various complex tasks~\citep{jaech2024openai,guo2025deepseek}. However, they still struggle with some simple logical reasoning tasks, and one of the most well-known failure modes is the ``reversal curse'', which refers to the phenomenon that when LLMs have learned a forward relation ``$A\to B$'' (e.g., \emph{Alice's husband is Bob.}) during training, they fail to answer the reverse question ``$ B\leftarrow A$'' (e.g., \emph{Who is Bob's wife?}) during test. 

Extensive prior works have attempted to understand or resolve the reversal curse. 
Recent research \citep{zhu2024towards, lin2024delving, wang2025reversal} suggests that this forward–reverse asymmetry is a fundamental generalization failure of autoregressive LLMs.
Consequently, existing mitigation strategies generally follow two paths: (i) augmenting or reformatting training examples to explicitly show the reverse links, and (ii) adjusting the learning objective or training procedure to reduce directional bias. However, such interventions often result in substantial deviation from standard data pipelines or training recipes and can introduce trade-offs in overall model quality.
Moreover, they commonly adopt the premise that, for causal autoregressive models, reliably answering the reversed queries ultimately requires that the specific direction be included in training.
\shawp{Revised}


In this paper, we challenge this view by showing that the reversal curse can be resolved by adding certain regularization to the training data, without being trained on reversal data. Specifically, we add the regularization by augmenting the training data with the ``Identity Bridge'', which was originally proposed by~\citet{lin2025identity} to solve two-hop reasoning tasks. The identity bridge refers to statements in the form of ``$A \to A$'' (e.g., \emph{The name of Alice is Alice}, or \emph{The wife of Alice's husband is Alice}). Semantically, it adds no new information to the dataset about the relations between entities, but it serves as a dataset-level regularization and can change the optimization landscape. In theory, we prove through the implicit bias of gradient descent that by adding the identity bridge regularization, even one-layer transformers can break the reversal curve on symbolic reasoning tasks. To further understand the mechanism of the identity bridge, we prove that the identity bridge regularized reversal task can be equivalently formulated as an out-of-context reasoning (OCR) problem~\citep{cohen2024evaluating}. Moreover, to validate our method in real-world scenarios, we conduct experiments with pretrained LLMs on real-world reversal tasks, and the trained model can achieve a pass rate as high as 50\%, which is a significant improvement over the previous near-zero accuracy.

In summary, the main contributions of this paper include:
\begin{itemize}
    \item We propose to use the identity bridge regularized data recipe to break the reversal curse in autoregressive LLMs. To the best of our knowledge, this is the first work that breaks the reversal curse in autoregressive LLMs without modifying the training paradigm (e.g., the model architecture or loss function) or reversing the training data.
    \item Theoretically, we prove through the implicit bias of gradient descent that even a one-layer transformer is able to break the reversal curse via identity bridge regularization, while without it, the reversal curse happens. We also prove that the identity bridge regularization is closely related to the OCR phenomenon.
    \item We conduct experiments on pretrained LLMs that achieve a 50\% pass rate on real-world reversal tasks via the identity bridge regularization, a significant improvement over the previous near-zero accuracy.
\end{itemize}

%% file: contents/prelim.tex
\section{Preliminaries}

\textbf{Basic notations.} Let $[N]=\{1,\cdots, N \}$. We denote $\be_i$ as one-hot vectors where only the $i$-th entry is non-zero and equals one, and denote $\bz_x$ as the embedding for token $x$. Let $\cV$ be the vocabulary set. We use $\bzero_{m\times n}$ to denote the $m \times n$ all-zero matrix and use $\bzero_d$ to denote the $d$-dimensional zero vector. $\cA$ and $\cB$ denote sets of entities and $\cR$ denotes the set of relations.
\paragraph{Reversal reasoning task.} Given two disjoint sets of $N$ entities $\cA:=\{ a_1,\cdots,a_N \}$ and $\cB:=\{ b_1,\cdots,b_N \}$, a forward relation $r_{+}$ is a bijection mapping $a_i$ to $b_i$ for all $i\in [N]$, and the reverse relation $r_{-}$ is defined as the inverse of $r_{+}$. The reversal reasoning task is to answer $r_{-}(b_i),i\in [N]$ when a model is only trained with the forward relations $r_{+}(a_i) = b_i, \ \forall i\in[N]$.
\paragraph{Dataset.} To adjust to the language model input format, we represent each relational instance $r(s)=s'$ as a token sequence $[s,r|s']$, where $s,s'$ are entities and $r$ is a relation. We treat $[s,r]$ as the input sequence and $s'$ as the target label. \shaw{E.g., if $s = $ Alice, $s' = $ Bob and $r =$ ``husband'', the sequence encodes ``Alice's husband is Bob''.} In this paper, we consider the following sets:
\begin{itemize}
    \item Forward relation set: $\cD_{r_{+}} = \{[a_i,r_{+}|b_i]: i\in [N]  \}$;
    \item Reversal relation set: $\cD_{r_{-}} = \{[b_i,r_{-}|a_i]: i\in [N]  \}$; 
    \item Identity bridge set: $\cD_\text{idn} = \{[a_i,r_\text{id}|a_i]: i\in [N]  \}\cup \{[b_i,r_\text{id}|b_i]: i\in [N]  \}$;
    \item Relation set: $\cR = \{ r_{+}, r_{-}, r_\text{id} \}$.
\end{itemize}

The identity bridge dataset $\cD_\text{idn}$, originally proposed by~\citet{lin2025identity} for two-hop reasoning tasks, contains an entity and an identity relation $r_\text{id}$ as input, and maps them to the entity itself. The identity bridge dataset actually contains no information, but will serve as a regularization to help break the reversal curse.

\paragraph{One-layer transformer.} We consider a one-layer decoder-only transformer $\TF$, which takes a sequence $x_{1:T} := (x_1, \ldots, x_T) \in \cV^{\mathrm{T}}$ as input and outputs a logit vector $\TF_{\btheta}(x_{1:T})\in \R^d$ as follows:
\begin{equation}
    \TF_{\btheta}(x_{1:T}) := \bWO \bWV^{\mathrm{T}} \bX \text{softmax}( \bX^{\mathrm{T}} \bWKQ \bx_T ),
\end{equation}
where  $\bX = [\bz_{x_1},\cdots,\bz_{x_T}] \in \mathbb{R}^{d\times T}$ is the embedding matrix of the tokenized sequence, $\bWO,\bWV\in \R^{d\times d_h}$ are the output and value matrices, respectively, $\bWKQ = \bW_\text{K} \bW_\text{Q}^{\mathrm{T}} \in \R^{d\times d}$ is the reparameterized key-query matrix, and $\btheta=(\bWO,\bWV,\bW_\text{KQ})$ encodes all model parameters. 

We further define the logit of token $y$ as
\begin{equation}
    \TF_{\btheta}(x_{1:T};y): = \bz_y^{\mathrm{T}} \TF_{\btheta}(x_{1:T}).
\end{equation}

The next token probability $p(v|x_{1:T})$ is computed as the softmax of the logit vector:
\begin{equation}
\label{eq:ntp}
    p_{\btheta}(v|x_{1:T}) = \frac{\exp (\TF_{\btheta}(x_{1:T};v))}{\sum_{v'\in \cV} \exp ( \TF_{\btheta}(x_{1:T};v'))}.
\end{equation}

\paragraph{Loss function.} We use the cross-entropy loss
\shawp{Added $\log$}
\begin{equation}\label{eq:loss}
    \cL_{\cD}(\btheta) = \E_{[x_{1:T}|y]\sim\cD} [- \log p_{\btheta}(y|x_{1:T})],
\end{equation}
where $\cD$ is a dataset. The model parameter will be optimized by running the gradient flow $\dot{\btheta} = - \nabla_{\btheta} \cL_{\cD}(\btheta)$ or gradient descent with sufficiently small learning rate.

%% file: contents/theory.tex
\section{Theoretical Results}
\label{sec:theoretical_res}

In this section, we study the reversal curse on a one-layer transformer theoretically. 

We first provide an explanation for the reversal curse in~\cref{sec:theory_explain}. A recent paper~\citep{zhu2024towards} also explains the reversal curse on a one-layer transformer by examining the learning dynamics, but their transformer model uses a non-factorized output and value matrix, \emph{i.e.}, a single reparameterized $\bW_\text{OV}$ instead of the factorized form $\bWO \bWV^{\mathrm{T}}$. As a complement to this result, we analyze the case for a factorized transformer as \shaw{prior works \citep{zhang2025training, huang2025generalization} indicate that the two parametrizations may have different optimization dynamics.}  
Moreover, our analysis is based on the implicit bias of gradient descent, providing a different perspective to understand this phenomenon. Then, in~\cref{sec:theory_idn_bridge}, we propose a simple data recipe called the Identity Bridge and prove that it can help break the reversal curse in one-layer transformers. Finally, we provide insight in~\cref{sec:theory_relation_OCR} to understand the mechanism of the identity bridge by relating to the out-of-context reasoning phenomenon~\citep{huang2025generalization}.

We take the embedding $\bz_{s}$ for all entities $s\in \cA \cup \cB$ to be one-hot, and set the embedding of the forward relation $\bz_{r_{+}}$ also to be a one-hot vector. For the embedding of the reversal relation $\bz_{r_{-}}$, we set it to be the negation of $\bz_{r_{+}}$, \emph{i.e.}, $\bz_{r_{-}} = - \bz_{r_{+}}$. 
This choice is not arbitrary: it idealizes a widely observed geometric structure in learned embedding spaces, where relational transformations correspond to approximately linear directions. For instance, classical results in word embeddings \citep{mikolov2013linguistic} show that relations can be captured by vector offsets (e.g., ``$\text{king } -\text{man} + \text{woman} \approx  \text{queen}$''). Under this lens, $\bz_{r_{-}} = - \bz_{r_{+}}$ can be interpreted as an abstraction of the inductive bias learned in pre-training. Forward and inverse relations are expected to occupy opposite directions in representation space, reflecting their mutually inverse semantics (e.g., “husband” vs. “wife”).

We fix the key-query matrix $\bW_\text{KQ}$ to $\bzero_{d\times d}$, so the trainable weights are the output and value matrices $\bWO,\bWV$, \emph{i.e.}, $\btheta=(\bWO,\bWV)$. This is reasonable since our prompt is in the form of $[s,r],s\in \cA \cup \cB , r\in \cR$, which is clean and free of noise, so the attention should pick both tokens. In this spirit, by fixing $\bW_\text{KQ}$ to zero matrix, tokens $s$ and $r$ will always have equal attention weights $\frac{1}{2}$. Moreover, we also require the intrinsic dimension $d_h \geq d$, in order to derive non-degenerate solutions. We formalize this as the following assumption.

\begin{assumption}\label{assum:fix_kq}
     The key-query matrix $\bW_\text{KQ}$ is fixed to $\bzero_{d\times d}$, and the intrinsic dimension $d_h$ satisfies $d_h \geq d$.
\end{assumption}

Note that this assumption is also effectively adopted by \citet{huang2025generalization}.
Under~\cref{assum:fix_kq}, the output logit is
\begin{equation}\label{eq:transformer_logit}
    \TF_{\btheta}([s,r]) = \frac{1}{2} \bWO \bWV^{\mathrm{T}} \bz_{s} + \frac{1}{2} \bWO \bWV^{\mathrm{T}} \bz_{r}.
\end{equation}

Given a model parameter $\btheta=(\bWO,\bWV)$, we further define the margin between the correct label $r(s)$ and incorrect label $s'\in \cA\cup\cB \backslash \{ r(s) \}$ as their logit difference:
\begin{equation}
    h_{[s,r],s'}(\btheta) =  \TF_{\btheta}([s,r];r(s) ) -  \TF_{\btheta}([s,r];s' ).
\end{equation}

Since $\TF_{\btheta}([s,r])$ only depends on $\bWO \bWV^{\mathrm{T}}$, we can also write the margin as $h_{[s,r],s'}(\bWO \bWV^{\mathrm{T}})$.

Now, we are ready to state a fundamental lemma established in~\citet{huang2025generalization}, which allows us to characterize $\bWO \bWV^{\mathrm{T}}$ under gradient descent as an SVM problem.

\begin{lemma}[SVM form, Theorem 1 in~\citet{huang2025generalization}]\label{lemma:SVM}
    Suppose~\cref{assum:fix_kq} holds. Consider running gradient descent with a small enough learning rate or gradient flow on loss~\eqref{eq:loss} where the model is defined by \eqref{eq:ntp} and \eqref{eq:transformer_logit}. If there exists a time $t_0$, such that $\cL_{\cD}(\btheta(t_0)) < 1$, then any limiting point of $\btheta / \| \btheta\|$, where $\btheta=(\bWO,\bWV)$, is along the direction of a KKT point of a program, which has the same solutions for $\bWOV^F = \bWO \bWV^{\mathrm{T}}$ as the following nuclear norm minimization problem
    \begin{equation}\label{eq:SVM}\tag{SVM}
        \begin{aligned}
        & \min_{\bWOV^{\textup{F}} } \ \   \frac{1}{2} \| \bWOV^{\textup{F}} \|_{\star}^2   \\ \text{s.t.} &\; h_{[s,r],s'}(\bWOV^{\textup{F}}) \geq 1,  \forall [s,r] \in \cD, \ \forall s' \in \cV \backslash \{r(s)\},
        \end{aligned}
    \end{equation}
    where $\| \cdot \|_{\star}$ denotes the nuclear norm.
\end{lemma}

\subsection{Implicit Bias Explains the Reversal Curse}
\label{sec:theory_explain}
In this part, we explain why the reversal curse happens from the lens of implicit bias by using~\cref{lemma:SVM}.

Suppose the training dataset contains only the forward relation data. Then we can show that the trained model fails to generalize to the reversal task, which can be formally summarized in the following theorem.
\begin{theorem}\label{thm:reversal_happen}
    Suppose~\cref{assum:fix_kq} holds, $N\geq 2$, and problem~\eqref{eq:SVM} with forward relation training set $\cD = \cD_{r_{+}}$ admits a single optimal solution $\bWOV^{\textup{+}}$ . Then, $\forall [b,r_{-}] \in \cD_{r_{-}}$, we have
    \begin{equation}
        h_{[b,r_{-}],a}(\bWOV^{\textup{+}}) = 0, \quad \forall a\in \cA \backslash \{ r_{-}(b) \}.
    \end{equation}
    Thus, the model can not generalize to the reversal relation.
\end{theorem}

\begin{proof}
    See~\cref{sec:proof_reversal_happen}.
\end{proof}

We visualize the solution $\bWOV^{\textup{+}}$ in~\cref{thm:reversal_happen} in~\cref{fig:svm_forward}. From~\cref{fig:svm_forward}, the solution has a positive diagonal in the lower left block, which means the model has learned the forward knowledge. To solve the reversal task, the model must simultaneously store knowledge in the upper-right block. However, training only on the forward relation dataset yields a zero upper-right block, leading to a zero margin and the reversal curse.

\subsection{Breaking the Reversal Curse via Identity Bridge}
\label{sec:theory_idn_bridge}
We have seen in~\cref{thm:reversal_happen} that training with the forward relation only dataset will suffer from the reversal curse. To resolve this problem, we propose to add regularization data to the training dataset, such that the model can encode the reversal knowledge in the weight via gradient.

The regularization data we use is called the ``Identity Bridge'', which was first introduced by~\citet{lin2025identity} to solve two-hop reasoning tasks. They showed that adding identity bridge data can enable the model to perform two-hop reasoning when only trained with one-hop data. In our paper, we take the spirit of this idea and use it to break the reversal curse.

\input{contents/figures/svm_forward}

The identity bridge dataset is in the form of
\begin{equation}
    \cD_\text{idn} = \{[a_i,r_\text{id}|a_i]: i\in [N]  \}\cup \{[b_i,r_\text{id}|b_i]: i\in [N]  \},
\end{equation}
where $r_\text{id}$ is an identity relation.

The identity bridge dataset only contains training samples that map an entity to itself. This identity mapping reveals no information about the reversal knowledge, but it will affect the gradient, thus changing the optimization landscape. 

We set the embedding of the identity relation to be a zero vector, \emph{i.e.}, $\bz_{r_\text{id}}=\bzero_d$. The reason for this choice is that the identity can be viewed as the composite of the forward relation and the reversal relation. Therefore, in this spirit, $\bz_{r_\text{id}} = \bz_{r_{-}} + \bz_{r_{+}}=\bzero_d$, since we have set $\bz_{r_{-}} = - \bz_{r_{+}}$.

Then, we have the following theorem, showing that the identity bridge regularization can break the reversal curse.
\begin{theorem}\label{thm:idn_bridge}
    Suppose~\cref{assum:fix_kq} holds, $N\geq 2$, and the problem~\eqref{eq:SVM} with identity bridge regularized training set $\cD = \cD_{r_{+}} \cup \cD_\text{idn}$ admits a single optimal solution $\bWOV^{\ast}$. Then, $\forall [b,r_{-}] \in \cD_{r_{-}}$, we have
    \begin{equation}
        h_{[b,r_{-}],a}(\bWOV^{\ast}) > 0, \forall a\in \cA \cup \cB \backslash \{ r_{-}(b) \}.
    \end{equation}

    Thus, the model can generalize to the reversal relation.
\end{theorem}

\begin{proof}
    See~\cref{sec:proof_idn_bridge}.
\end{proof}
We visualize the $\bWOV^{\ast}$ solution in~\cref{thm:idn_bridge} with the identity bridge regularized training dataset in~\cref{fig:svm_idn}. We note that, in this case, the solution has a positive upper-right block, which means the model has encoded the reversal knowledge in the weights. 

\input{contents/figures/svm_identity_bridge}

An intuitive explanation of why the identity bridge can help is that the identity training data forces the model to have positive weights at its two diagonal blocks, as observed in~\cref{fig:svm_idn}, and to minimize the nuclear norm of $\bW_{OV}^\ast$, the upper-right block is required to have positive diagonal values.

In addition to this explanation, we find that the identity bridge is closely related to the following out-of-context reasoning phenomenon, which can provide more insight into the working mechanism of our method.

\subsection{Relation to Out-of-Context Reasoning}
\label{sec:theory_relation_OCR}
\begin{figure*}[t]
  \begin{center}
    \centerline{\includegraphics[width= 1.8 \columnwidth]{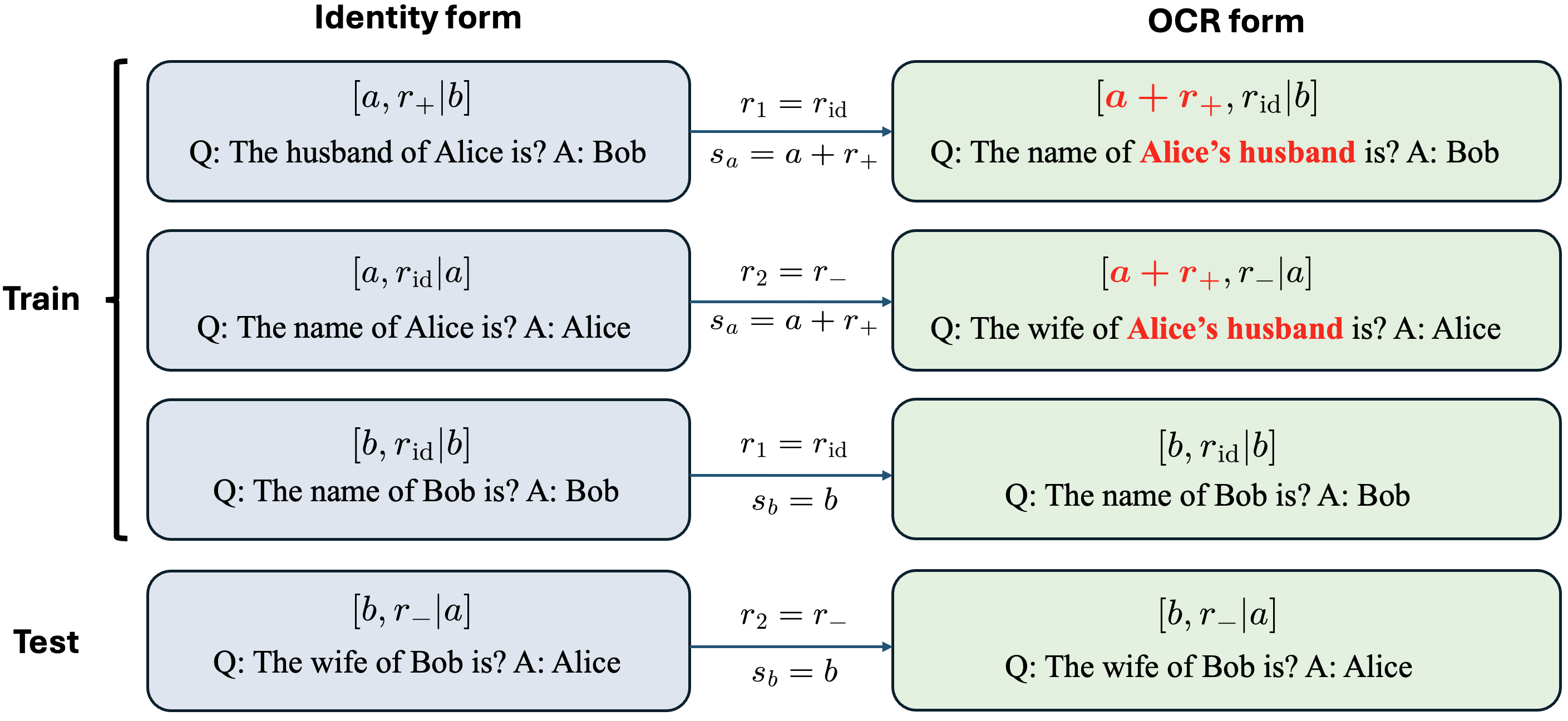}}
    \caption{
      Illustration of~\cref{prop:relation_OCR} when $\bz_{r_{+}} = - \bz_{r_{-}}$. Identity form refers to the form of the identity regularized dataset, and OCR form refers to its OCR form given in~\cref{prop:relation_OCR}. In the concrete ``Husband-Wife'' example in the Identity form, we set $a$ to ``Alice'', $b$ to ``Bob'', $r_{+}$ to ``husband'', $r_{-}$ to ``wife'', and $r_\text{id}$ to ``name''. In the corresponding OCR form, $s_a = a+ r_{+}$ is ``Alice's husband'', $s_b = b$ is Bob, $r_1 = r_\text{id}$ is ``name'', and $r_2 = r_{-}$ is ``wife''. The left Identity form and the right OCR form are semantically equivalent and share the same test dataset.
    }
    \label{fig:illus_ocr}
  \end{center}
\end{figure*}

Out-of-context reasoning (OCR)~\citep{cohen2024evaluating,huang2025generalization} refers to a model’s ability to deduce implications beyond the explicitly trained knowledge by drawing connections between different pieces of learned knowledge. Following \citet{huang2025generalization}, in an OCR task, one has a training subject set $\cS_\text{train}$, a test subject set $\cS_\text{test}$, two relations $r_1,r_2$, a set of facts $\cF$, and a set of implications $\cI$. Relation $r_1$ maps a subject to a fact in $\cF$ and $r_2$ maps a subject to an implication in $\cI$, such that $\forall s_1, s_2 \in \cS_\text{train} \cup \cS_\text{test}$, if they share the same fact $r_1(s_1) = r_1(s_2)$, then they have the same implication $r_2(s_1) = r_2 (s_2)$. An OCR task is to train the model with the dataset $\cD_\text{train} = \{ [s,r_1|r_1(s)]: s\in  \cS_\text{train} \cup \cS_\text{test} \} \cup \{ [s,r_2|r_2(s)]: s\in  \cS_\text{train} \}$ which contains implications on $\cS_{\text{train}}$ and all facts, and test the implications on $\cS_\text{test}$, \emph{i.e.}, $\cD_\text{test} = \{ [s,r_2|r_2(s)]: s\in  \cS_\text{test} \}$. For example, if we take $r_1$ to be ``lives in'', and $r_2$ to be ``speaks'', then the training dataset looks like $\{ $ [``Jone'' ,``lives in'' $|$ ``Japan''], [``Jone'', ``speaks'' $|$ ``Japanese'' ], [``Mike'',``lives in'' $|$ ``Japan''] $ \}$, and we expect the model can deduce ``Mike speaks Japanese''.

We show in the following proposition that the identity bridge regularized reversal task is equivalent to an OCR task if we select an appropriate identity relation embedding.  
\begin{proposition}\label{prop:relation_OCR}
    For arbitrary embedding of entities $\bz_{s},s\in \cA \cup\cB$ and forward and reversal relations $\bz_{r_{+}},\bz_{r_{-}}$, if the identity relation embedding is selected as $\bz_{r_\text{id}} = (\bz_{r_{+}} + \bz_{r_{-}})/2$, then there exists an OCR task equivalent to the regularized reversal task in the sense that they have the same test dataset and their training datasets lead to the same SVM problem~\eqref{eq:SVM}.
\end{proposition}

\begin{proof}
    The idea is to use the linearity of the one-layer transformer model~\eqref{eq:transformer_logit} and rewrite the input sequence $[s,r]$ to an appropriate form. Specifically, we construct the OCR task as follows.\allowdisplaybreaks
    \begin{align}\allowdisplaybreaks
        \cS_\text{train} =& \{ s_a^{(i)}:i\in [N] \}, \bz_{s_a^{(i)}} = \bz_{a_i} + \frac{\bz_{r_{+}} - \bz_{r_{-}}}{2} \\
         \cS_\text{test} = &\{ s_b^{(i)}:i\in [N] \}, \bz_{s_b^{(i)}} =\bz_{b_i} \\
         \bz_{r_1} =& \bz_{r_\text{id}}, \bz_{r_2} = \bz_{r_{-}}\\
         \cF = & \{ u_b^{(i)}: i\in [N]\} , \bz_{u_b^{(i)}} = \bz_{b_i} \\
        \cI = &  \{ u_a^{(i)}: i\in [N] \}, \bz_{u_a^{(i)}} =\bz_{a_i}\\
         r_1(s_a^{(i)}) = & u_b^{(i)}, r_2(s_a^{(i)}) = u_a^{(i)}\\
         r_1(s_b^{(i)}) = & u_b^{(i)}, r_2(s_b^{(i)}) = u_a^{(i)}\\
             \cD_\text{train-OCR}  
         = & \{ [s_a^{(i)},r_1|u_b^{(i)}], [s_a^{(i)},r_2|u_a^{(i)}], [s_b^{(i)},r_1|u_b^{(i)}] : i \in [N] \}
        \\
         \cD_\text{test-OCR} =& \{ [s_b^{(i)},r_2|u_a^{(i)}] : i \in [N] \}.
    \end{align}

    In this construction, the test dataset $\cD_\text{test-OCR}$ is the same as that of the reversal task. To verify the equivalence of the SVM form, we only need to check the sum of the subject embedding and the relation embedding. Specifically, we have
    \begin{align}
        \bz_{s_a^{(i)}} + \bz_{r_1} = \bz_{a_i} + \bz_{r_{+}}\\
        \bz_{s_a^{(i)}} + \bz_{r_2} = \bz_{a_i} + \bz_{r_\text{id}}\\
        \bz_{s_b^{(i)}} + \bz_{r_1} = \bz_{b_i} + \bz_{r_\text{id}}.
    \end{align}

    Therefore, $\cD_\text{train-OCR}$ leads to the same SVM problem as the identity bridge regularized dataset $\cD_{r_{+}}\cup \cD_\text{idn}$.
\end{proof}

We give an illustration of~\cref{prop:relation_OCR} in~\cref{fig:illus_ocr}, where we assume $\bz_{r_{+}} = - \bz_{r_{-}}$, so that $\bz_{s_a^{(i)}} = \bz_{a_i} + \frac{\bz_{r_{+}} - \bz_{r_{-}}}{2} = \bz_{a_i} + \bz_{r_{+}}$. In~\cref{fig:illus_ocr}, we give a concrete ``Husband-Wife'' example to show how this form equivalence works in real-world language data. In the example, we set $a$ to ``Alice'', $b$ to ``Bob'', $r_{+}$ to ``husband'', $r_{-}$ to ``wife'', and $r_\text{id}$ to ``name''. The trick in the form transformation is using the equality $r_\text{id} = r_{+} + r_{-}$, \emph{i.e.}, ``name'' = ``wife'' + ``husband'', and setting $s_a = a+ r_{+}$, \emph{i.e.}, ``Alice's husband''. For other parts in the corresponding OCR form, $s_b = b$ is Bob, $r_1 = r_\text{id}$ is ``name'', and $r_2 = r_{-}$ is ``wife''. The Identity form and the OCR form are semantically equivalent and share the same test dataset.

We note that the identity relation embedding used in~\cref{sec:theory_idn_bridge} also satisfies the condition in~\cref{prop:relation_OCR}, and \cref{prop:relation_OCR} provides another explanation of why the identity bridge regularization can help break the reversal curse from the perspective of OCR. 

\cref{prop:relation_OCR} allows us to transform a reversal task into an OCR task, and we will show in the experiment that this transformation is critical to break the reversal curse in real LLMs.

%% file: contents/figures/svm_forward.tex
\begin{figure}[tbhp!]
     \centering
     \begin{minipage}{0.45\textwidth}
         \centering
         \includegraphics[width= 2\columnwidth]{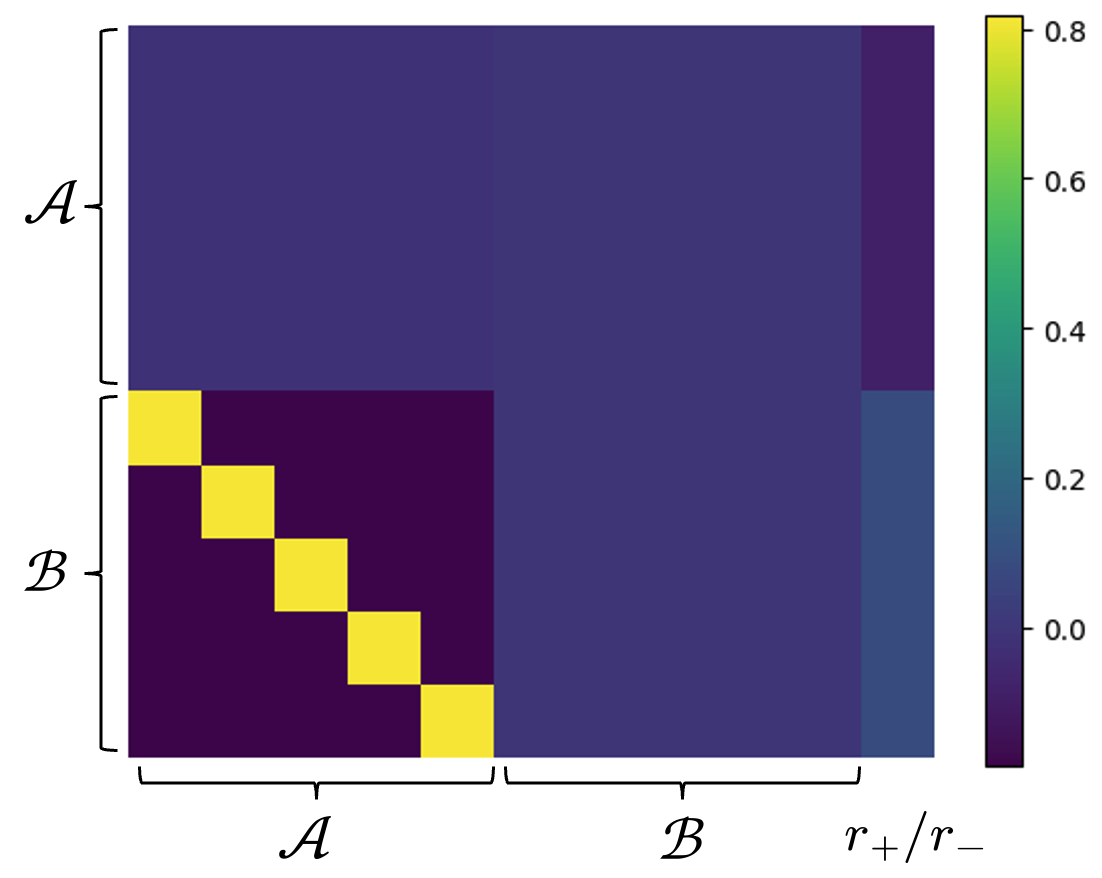}
         \caption{$\bW_{OV}^{\textup{+}}$ solution in~\cref{thm:reversal_happen} with forward relation dataset. In $\bW_{OV}^{\textup{+}}$, the diagonal weight of the upper right block is equal to the off-diagonal weight of it, which means when tested with the reversal data $[b_i,r_{-}]$, the trained model will output equal logits over all $a\in \cA$. Thus, only training with the forward relation dataset will lead to the reversal curse.}
         \label{fig:svm_forward}
     \end{minipage}
     \hfill 
     \begin{minipage}{0.45\textwidth}
         \centering
         \includegraphics[width= 2 \columnwidth]{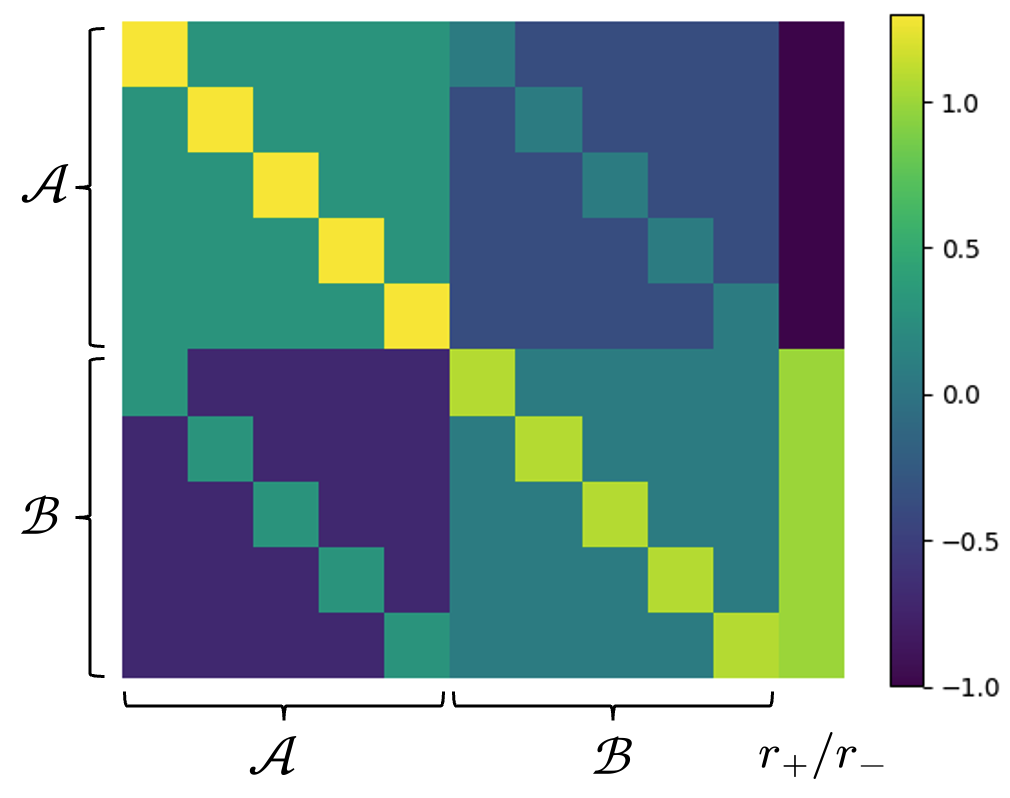}
         \caption{      $\bW_{OV}^\ast$ solution in~\cref{thm:idn_bridge} with identity bridge regularized dataset. In $\bW_{OV}^{\ast}$, the diagonal weight of the upper right block is larger than the off-diagonal weight of it, so the logits of the correct reversal answer $a_i$ will be larger than other answers when tested with $[b_i,r_{-}]$. Thus, the model can break the reversal curse with identity bridge regularized dataset.}
         \label{fig:svm_idn}
     \end{minipage}
\end{figure}

%% file: contents/figures/svm_identity_bridge.tex


%% file: contents/experiments.tex
\section{Experiments}
In this section, we conduct experiments on both one-layer transformers and pretrained real-world LLMs to validate our method. In~\cref{sec:exp_one_layer_tf}, we present experiments on one-layer transformers to verify the theoretical results in~\cref{sec:theoretical_res}. Then, in~\cref{sec:exp_real_llm}, we conduct experiments on real LLMs to solve real-world reversal tasks. Detailed experiment setups and additional experimental results can be found in~\cref{app:experiment_detail}.

\subsection{One-Layer Transformer Experiments}
\label{sec:exp_one_layer_tf}
We experiment with the reversal task with $N=10$ pairs, and the embeddings are set as $\bz_{a_i} = \be_{i},\bz_{b_i} = \be_{N+i},\bz_{r_{+}} = \be_{2N+1}, \bz_{r_{-}} = - \be_{2N+1}, \bz_{r_\text{id}} = \bzero_d$. 

The reversal test loss and mean reciprocal rank (MRR) of forward relation only training data and identity bridge regularized data are presented in~\cref{fig:one_layer_curve}. From~\cref{fig:one_layer_curve}, we observe that, without the identity bridge, the reversal test MRR stays around the initialization level, while adding identity bridge regularization, the model can generalize to all the reversal tests.
\begin{figure}[t]
  \vskip 0.2in
  \begin{center}
    \centerline{\includegraphics[width=  1 \columnwidth]{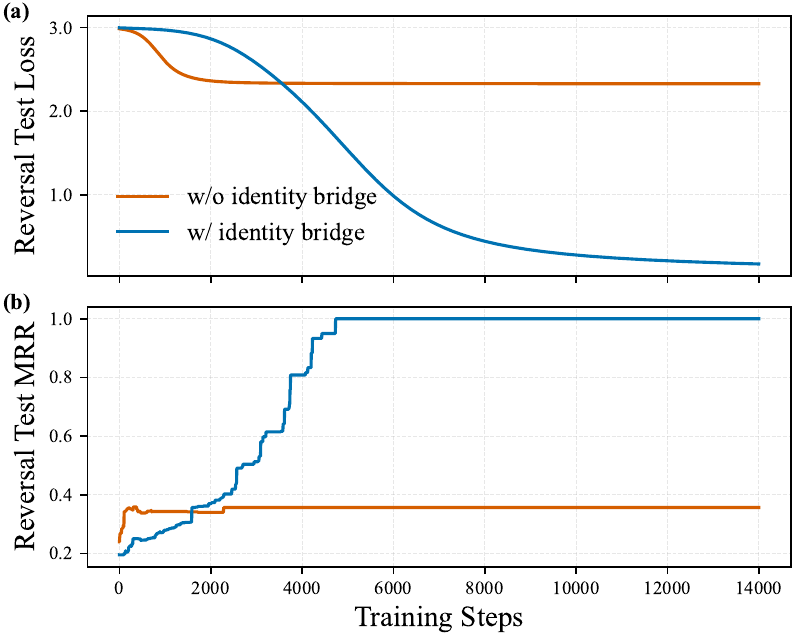}}
    \caption{
      {Reversal test loss and mean reciprocal rank (MRR) of forward relation only training data vs. identity bridge regularized data.} Without the identity bridge, the model stays around the initialization level, while the model can generalize to all the reversal tests after adding the identity bridge regularization.
    }
    \label{fig:one_layer_curve}
  \end{center}
  \vspace{-20pt}
\end{figure}

We also present in~\cref{fig:train_forward_tf} the $\bWO\bWV^{\mathrm{T}}$ weight after training with the forward relation dataset and the identity bridge regularized dataset, respectively. The trained $\bWO\bWV^{\mathrm{T}}$ weights match the theoretical results in~\cref{thm:reversal_happen} and~\cref{thm:idn_bridge}.

\begin{figure}[htbp!]
     \centering
     \begin{minipage}{0.48\textwidth}
         \centering 
         \includegraphics[width=2.5 \columnwidth]{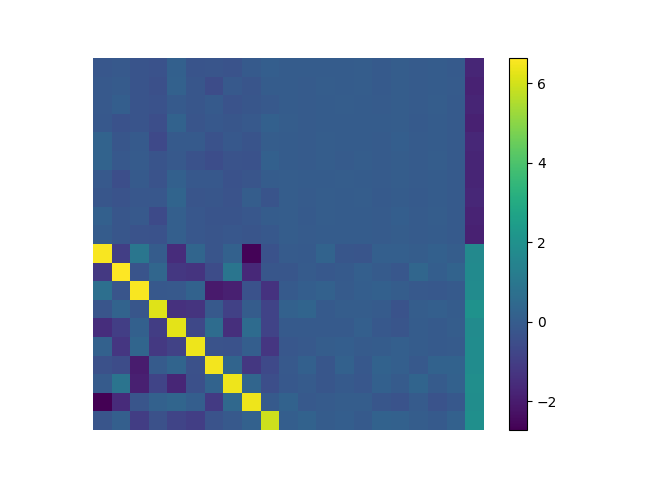}
         \sublabel{a}
     \end{minipage}
     \hfill 
     \begin{minipage}{0.48\textwidth}
         \centering
         \includegraphics[width=2.5 \columnwidth]{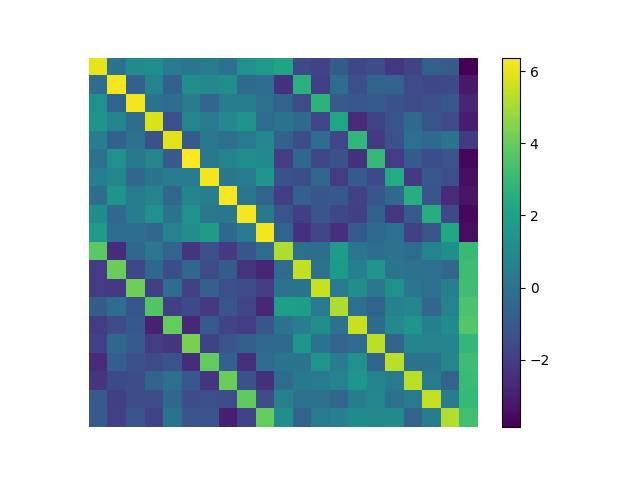}
         \sublabel{b}
     \end{minipage}
     \caption{(a). $\bWO\bWV^{\mathrm{T}}$ weight after training with the forward relation dataset; (b) $\bWO\bWV^{\mathrm{T}}$ weight after training with the identity bridge regularized dataset. The trained weights match the SVM solutions in \Cref{thm:reversal_happen} and \Cref{thm:idn_bridge}.}
     \label{fig:train_forward_tf}
\end{figure}



\subsection{Real Large Language Model Experiments}
\label{sec:exp_real_llm}
In this part, we consider two real-world reversal tasks, ``Husband-wife'' and ``Parent-Child''. In each task, the entity names are drawn from real-life names, and we randomly associate them to form $100$ reversal pairs. We fine-tune a Llama-3.2-1B-Instruct model to solve these tasks. Each experiment is run on three random seeds.

The training dataset and test dataset contain prompts in the following  (the same as in~\cref{prop:relation_OCR} and~\cref{fig:illus_ocr}), where we take the ``Husband-wife'' task for an example.

Training dataset:
\begin{enumerate}
    \item $[a+r_{+},r_\text{id}|b]$: Q: The name of Alice's husband is? A: Bob.\label{item:prompt_a_to_b} (Repeated $k$ times)
    \item $[a+r_{+},r_{-}|a]$: Q: The wife of Alice's husband is? A: Alice.\label{item:prompt_a_to_a}
    \item $[b,r_\text{id}|b]$: Q: The name of Bob is? A: Bob. \label{item:prompt_b_to_b}
\end{enumerate}

Test dataset:
\begin{itemize}
    \item $[b,r_{-}|a]$: Q: The wife of Bob is? A: Alice.
\end{itemize}

In the identity bridge data for $a$, instead of using a prompt like ``The name of Alice is?'', we rewrite it by decomposing the identity relation ``The name of'' to be the composition of the forward and the reversal relation, \emph{i.e.}, ``The wife of $a$'s husband''. This rewriting trick transforms the reversal task into its OCR form (See~\cref{prop:relation_OCR}). In this equivalent OCR task, the training subject ``Alice's husband'' (\emph{i.e.}, $a+r_+$) has two relations, ``name'' and ``wife''. The test subject is ``Bob'', which has seen one of the relation ``name'' in the training set, and the task is to generalize to the other relation ``wife''.

In the forward relation~\ref{item:prompt_a_to_b}, we duplicate this data $k$ times in the training, and our experiments (\cref{sec:exp_ablation_dataset} and~\cref{fig:dataset_compare}) show this is important for generalization.

The reversal test results of the ``Husband-Wife'' task and the ``Parent-Child'' task are presented in~\cref{fig:task_compare}, where the OCR stands for the proposed method with the training dataset introduced before, where we take $k=6$, and FWD represents the forward relation-only dataset (only ``Q: The husband of Alice is? A: Bob.''). In the experiment, we test two metrics: reversal test accuracy and reversal test loss. The reversal test accuracy is the fraction of reversal questions the model answers correctly, and the reversal test loss is the next-token prediction loss on the ground-truth name answer. For each task, both FWD and our OCR method can optimize the training loss to zero. However, the FWD method can not generalize to the reversal relation, and the reversal test accuracy stays zero. On the contrary, our method augments the dataset with identity bridge data and can successfully break the reversal curse, where the reversal test accuracy keeps growing to nearly 50\%.


\begin{figure}[thpb]
  \vspace{-20pt}
  \begin{center}
    \centerline{\includegraphics[width= 1 \columnwidth]{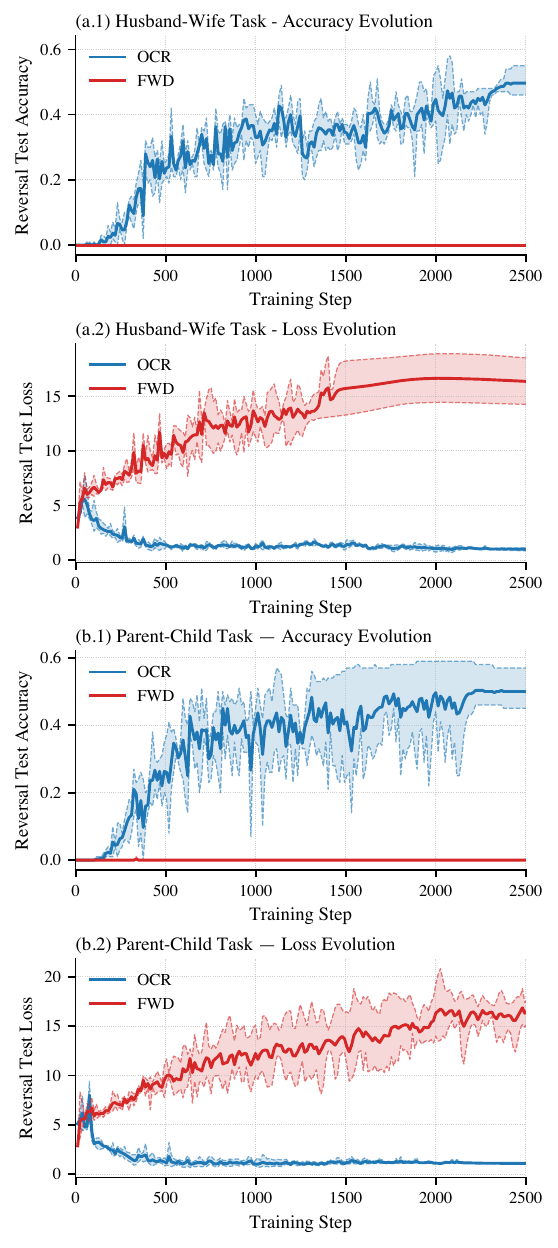}}
    \caption{
      Reversal test results of the ``Husband-Wife'' task (a.1 and a.2) and the ``Parent-Child'' task (b.1 and b.2): OCR stands for the proposed method using identity bridge augmented data (\emph{i.e.},~\ref{item:prompt_a_to_b}$\times 6$+\ref{item:prompt_a_to_a}+\ref{item:prompt_b_to_b}), and FWD represents the forward relation-only dataset (only ``Q: The husband of Alice is? A: Bob.''). The reversal test accuracy of the FWD dataset stays around 0, while our identity bridge regularized dataset OCR can break the reversal curse to around 50\% of the reversal questions in both tasks.
    }
    \label{fig:task_compare}
  \end{center}
   \vspace{-30pt}
\end{figure}

\begin{figure*}[t]
  \begin{center}
    \centerline{\includegraphics[width=2\columnwidth]{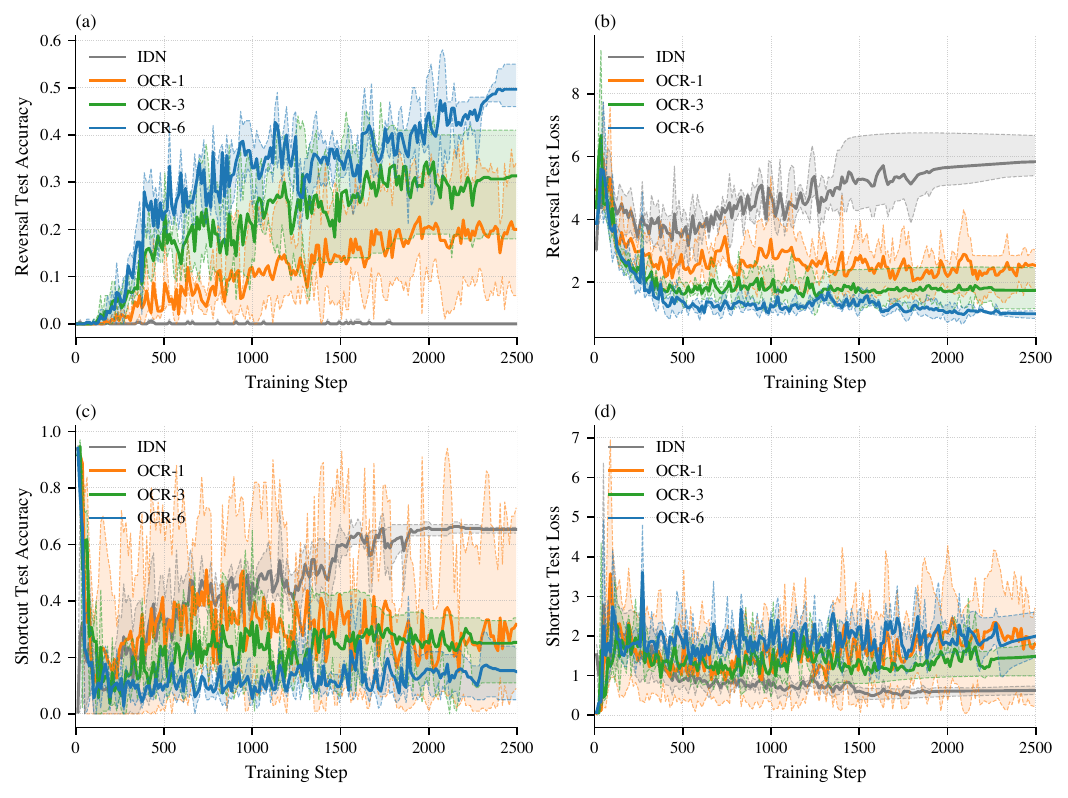}}
    \caption{
      Test results for different identity bridge formats. (a). Reversal test accuracy; (b). Reversal test loss; (c). Shortcut test accuracy; (d). Shortcut test loss. \shaw{(c) and (d) test whether the model learns the shortcut from the identity bridge data to directly copy the subject's name.}
      With the IDN dataset, the model never learns the reversal knowledge, while in the OCR form case, the model can generalize to the reversal relation. In the OCR form, repeated forward relation data greatly helps to generalize to the reversal relation, with the OCR-$6$ dataset reaching the highest reversal test accuracy of 50\%.
    }
    \label{fig:dataset_compare}
  \end{center}
\end{figure*}

\subsubsection{Ablation Experiments on Identity Bridge Format}
\label{sec:exp_ablation_dataset}
In this part, we focus on the ``Husband-Wife'' task and investigate the effect of the format of the identity bridge data on the model's reversal generalization ability. According to~\cref{prop:relation_OCR}, the identity bridge regularized dataset has the following IDN format and OCR format:
\begin{itemize}
    \item Identity (IDN) format dataset
        \begin{itemize}
            \item Q: The husband of Alice is? A: Bob.
            \item Q: The name of Alice is? A: Alice.
            \item Q: The name of Bob is? A: Bob.
        \end{itemize}
    \item OCR-$k$ format dataset:~\ref{item:prompt_a_to_b} $\times k$ +~\ref{item:prompt_a_to_a} +~\ref{item:prompt_b_to_b}.
\end{itemize}


In the IDN dataset, the identity relation of $a$ is directly phrased as ``The name of a is a'', instead of using the rephrasing trick as in~\ref{item:prompt_a_to_a}. Although these two formats are semantically equivalent, their generalization ability is fundamentally different.

Except for the reversal test data, we further consider the following shortcut test data
\begin{itemize}
    \item $[b,r_{-}|b]$: Q: The wife of Bob is? A: Bob.
\end{itemize}

The shortcut test data measures whether the model learns a shortcut to copy the name token after ``wife'' into its output, which can be learned from data~\ref{item:prompt_a_to_a}, resulting in a prediction that contradicts the ground-truth knowledge.

The test results of these datasets are shown in~\cref{fig:dataset_compare}. Across all datasets, the shortcut test accuracy quickly rises to near 100\% within a few gradient steps, then decrease to near 10\% in the first 200 steps, and finally slowly grows over the remaining steps. This indicates that the model learns the shortcut much faster than the reversal knowledge, and then the shortcut learning compete with the reversal relation learning. With the IDN dataset, the model never learns the reversal knowledge and the shortcut test accuracy reached to around 60\%. While under the OCR form, the model is able to break the reversal curse (as shown in the OCR dataset case), and the repeated forward relation data greatly helps to generalize to the reversal relation, with the OCR-$6$ dataset reaching the highest reversal test accuracy of 50\%. This also indicates that we should make identity bridge function as a regularization and the main learning gradient should come from the informative forward relation data.
Together, these results suggest that the identity bridge does not work {without reforming to its OCR form}, which is ignored by other works studying the identity bridge~\citep{lin2025identity}.

\subsubsection{Ablation Experiments on Entity Token Length}
In this part, we investigate the influence of entity token length on the model's reversal generalization ability. All the experiments are on the ``Husband-Wife" task. We consider three types of names: number name (1-token, e.g., ``34''), normal name (2-token, e.g., ``Sophia''), and long name (3-token, e.g., ``Catalina''). The reversal test results are presented in~\cref{fig:token_compare}. From~\cref{fig:token_compare}, a surprising fact is that when the name is a one-token number, the reversal test accuracy can reach nearly 100\%. But when the name is three tokens long, the test accuracy decreases to 10\%. This result suggests that entities with fewer tokens can be learned better in the reversal task, which is also observed in~\citet{wang2025reversal}.

\begin{figure}[t]
  \begin{center}
    \centerline{\includegraphics[width= 1 \columnwidth]{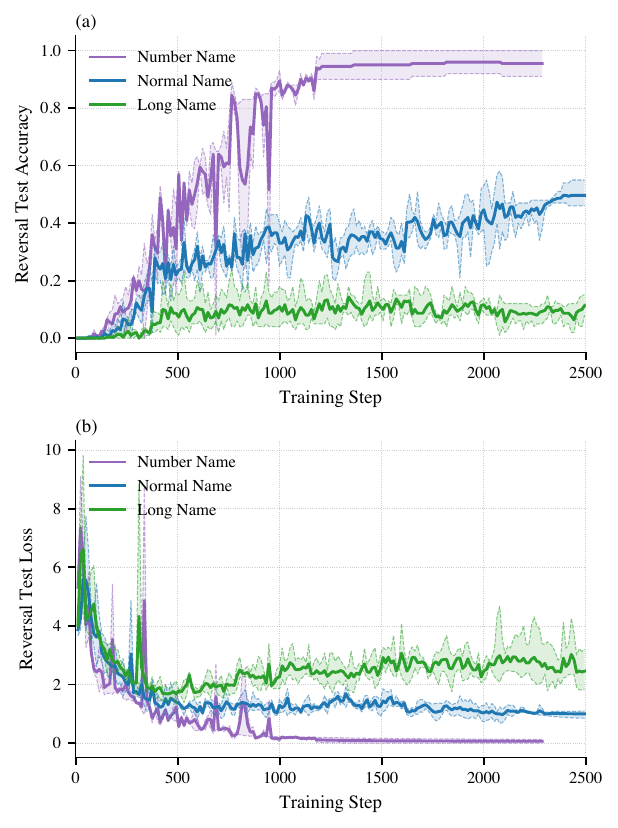}}
    \caption{
      Reversal test results for different token lengths. (a). Reversal test accuracy; (b). Reversal test loss. The model generalizes better on shorter names, even reaching 100\% accuracy on one-token number names.
    }
    \label{fig:token_compare}
  \end{center}
\end{figure}

\subsubsection{Ablation Experiments on Weight Decay}

In this part, we investigate the effect of weight decay.~\citet{kobayashi2024weight} showed that weight decay has an implicit bias toward minimum nuclear norm over factorized matrices, which has same effect as~\cref{lemma:SVM} to facilitate reversal generalization. Empirically, we also observe this effect in the experiment. The ablation result of reversal test accuracy across weight decay values are shown in~\cref{tab:wd-sweep}. From~\cref{tab:wd-sweep}, the reversal generation is weak at both small weight decay (e.g., 0.25) and large weight decay (e.g., 0.45). While in the active band, the weight decay greatly helps the generalization, with the highest reversal test accuracy of 55\%, which is the value of take in the main experiments.

\begin{table}[htpb!]
  \centering
  \caption{Mean reversal test accuracy across weight decay values.}
  \label{tab:wd-sweep}
  \begin{tabular}{cc}
  \toprule
  Weight Decay & Reversal Test Accuracy \\
  \midrule
  0.15 & 0.12 \\
  0.20 & 0.50 \\
  0.25 & 0.48 \\
  0.30 & 0.55 \\
  0.35 & 0.43 \\
  0.40 & 0.42 \\
  0.45 & 0.33 \\
  \bottomrule
  \end{tabular}
  \end{table}

%% file: contents/related_work.tex
\section{Related Work}

\paragraph{Reversal curse.} 
The reversal curse phenomenon was first reported by~\citet{berglund2023reversal}. Extensive prior works have attempted to understand or resolve the reversal curse. \citet{zhu2024towards} analyzed the training dynamics of different autoregressive models and showed that the reversal curse is caused by the model weights asymmetry. \citet{lin2024delving} suggested that the reversal curse stems from the factual recall bias in LLMs. \citet{wang2025reversal} conjectured that it is the inconsistency and entanglements of concept representations in LLMs incurs the reversal curse.
To mitigate the reversal curse, several works~\citep{golovneva2024reverse, guo2024mitigating, lu2024rethinking, pan2025closing} proposed data-centric interventions such as reordering or paraphrasing training examples, or training on masked variants, while others~\citep{lv2024analysis,kitouni2024factorization} attributed the issue to the training objective itself and attempt to address it by modifying the learning formulation. However, these methods significantly change the existing training paradigms, which can harm the model performance. Separately, \citet{nie2025large} finds that diffusion language models trained with a bidirectional denoising objective is able to break the reversal curse, further suggesting that the issue is closely tied to the autoregressive training paradigms. These works all failed to resolve the reversal curse in a reasonable way, leaving it a seemly impossible task. But our work challenges this view and breaks the reversal curse just by adding regularization to the training data.
\paragraph{Training dynamics of transformers.} 
A growing body of literature examines the optimization of transformer-based models~\citep{mahankali2023one,fu2023can,tian2023scan,tian2023joma,zhang2024trained,li2024mechanics,huang2024context,guo2024active}. 
In particular, recent works focus on understanding the transformer's behavior on various reasoning tasks through the lens of training dynamics. 
For example, previous studies have explored the emergence of induction heads~\citep{boix2023transformers}, factual recall~\citep{nichani2024transformers}, chain-of-thought reasoning~\citep{wen2024sparse, huang2025transformers}, chain of continuous thought~\citep{zhu2025reasoning,zhu2025emergence}, multi-hop reasoning~\citep{guo2025llms,wang2025learning}, etc. 
Our theoretical analysis mainly builds on \citet{huang2025generalization}, which providedsss a theoretical understanding of out-of-context reasoning. Contrary to the negative result in \citet{zhu2024towards}, we show that even a one-layer transformer is able to learn the reversal reasoning with the identity bridge.

\paragraph{Implicit bias. } The implicit bias of gradient descent has been extensively studied in classification tasks, which connects problems with logistic or exponentially-tailed loss to margin maximization \citep{soudry2018implicit,gunasekar2018implicit,gunasekar2018geo,lyu2019gradient, nacson2019convergence,nacson2019lexicographic,ji2019nonsep,vardi2022margin}. There are also many works exploring this connection in attention-based models \citep{tarzanagh2023transformers,ataee2023max,li2024mechanics,ildiz2024self,sheen2024implicit, vasudeva2024implicit}. Similar to \citet{huang2025generalization}, our work characterizes the solution to SVM programs to understand how the identity bridge helps break the reversal curse in transformers.

%% file: contents/conclusion.tex
\section{Conclusions}

In this paper, we propose using the identity bridge regularized data recipe to break the reversal curse, which, to the best of our knowledge, is the first method that can work for autoregressive LLMs without modifying the training paradigm or reversing the training data. We prove in theory that this data recipe can help break the reversal curse even in a one-layer transformer, and relate it to the OCR phenomenon, which provides insight for understanding the identity bridge regularization. In experiments, a 1B pretrained language model finetuned with the proposed data recipe achieves a 50\% success rate on reversal tasks, fully validating the effectiveness of the proposed method. The proposed identity bridge regularization method also admits limitations that the model can learn a shortcut, which is hard to eliminated, resulting in a gap between the 50\% pass rate and the ideal 100\% pass rate. For future directions, it would be helpful to understand the gap between one-token entity and multi-token entity, or symbolic data and textual token data, for real LLMs to solve the reversal tasks.

\FloatBarrier

%% file: contents/ack.tex
\section*{Acknowledgements}

This work was supported by the U.S. Army Research Laboratory and the U.S. Army Research Office under Grant W911NF2010219, Office of Naval Research, and NSF. This work used Jetstream2 at Indiana University through allocation CIS251312 from the Advanced Cyberinfrastructure Coordination Ecosystem: Services \& Support (ACCESS) program, which is supported by National Science Foundation grants \#2138259, \#2138286, \#2138307, \#2137603, and \#2138296.

%% file: appendix/appendix.tex
\input{appendix/appendix_proofs}

\input{appendix/appendix_experiments}

%% file: appendix/appendix_proofs.tex
\section{Omitted Proof}
\subsection{Proof of~\cref{thm:reversal_happen}}
\label{sec:proof_reversal_happen}

Without loss of generality, we assume the embedding of $a_i\in \cA,b_i\in\cB,r_{+}$ to be $\bz_{a_i} = \be_{i},\bz_{b_i} = \be_{N+i},\bz_{r_{+}} = \be_{2N+1}$, respectively. Then, we can obtain the following lemma characterizing the symmetry of the solution $\bWOV$.
\begin{lemma}\label{lemma:symmetry_solution_no_idn}
    Suppose $\bWOV$ is the single solution to problem~\eqref{eq:SVM} with $\cD = \cD_{r_{+}}$. Then,  $\bWOV$ admits the following form
    \begin{equation}\label{eq:symmetry_solution_no_idn}
        \bWOV =
\left[
\begin{matrix}
    f_1 \bI_N + f_2 \bE_N & g_1 \bI_N + g_2 \bE_N & \beta \bone_N \\
    l_1 \bI_N + l_2 \bE_N & m_1 \bI_N + m_2 \bE_N & \alpha \bone_N
\end{matrix}
\right],
    \end{equation}
where $f_1,f_2,g_1,g_2,l_1,l_2,m_1,m_2,\beta,\alpha$ are constants.

The singular values of $\bWOV$ in~\eqref{eq:symmetry_solution_no_idn} are 
\begin{equation}
    \{ \sigma_1^{(1)},\sigma_1^{(2)},\underbrace{\sigma_2^{(1)},\cdots,\sigma_2^{(1)}}_{N-1},\underbrace{\sigma_2^{(2)},\cdots,\sigma_2^{(2)}}_{N-1} \},
\end{equation}
and they can be computed by
\begin{align}
    & \sigma_1^{(1)} = \sqrt{\frac{(C_{A1} + NC_{A2} + C_{D1} + NC_{D2}) + \sqrt{(C_{A1} + NC_{A2} - (C_{D1} + NC_{D2}))^2 + 4(C_{B1} + NC_{B2})^2}}{2}};\\
    & \sigma_1^{(2)} = \sqrt{\frac{(C_{A1} + NC_{A2} + C_{D1} + NC_{D2}) - \sqrt{(C_{A1} + NC_{A2} - (C_{D1} + NC_{D2}))^2 + 4(C_{B1} + NC_{B2})^2}}{2}}; \\
    & \sigma_2^{(1)} =  \sqrt{\frac{(C_{A1} + C_{D1}) + \sqrt{(C_{A1} - C_{D1})^2 + 4C_{B1}^2}}{2}}; \\
    &\sigma_2^{(2)} =  \sqrt{\frac{(C_{A1} + C_{D1}) - \sqrt{(C_{A1} - C_{D1})^2 + 4C_{B1}^2}}{2}},
\end{align}
where
\begin{align}
& C_{A 1}= f_1^2+ g_1^2 \\
& C_{A 2}= 2 f_1 f_2+ N f_2^2+ 2 g_1 g_2+ N g_2^2 + \beta^2 \\
& C_{D 1}= l_1^2+ m_1^2 \\
& C_{D 2}=2 l_1 l_2+N l_2^2+ 2 m_1 m_2+N m_2^2 + \alpha^2 \\
& C_{B 1}= f_1 l_1+ g_1 m_1 \\
& C_{B 2}=f_1 l_2+f_2 l_1+N f_2 l_2+g_1 m_2+g_2 m_1+ N g_2 m_2 + \beta \alpha .
\end{align}
\end{lemma}

The proof of this lemma is almost the same as that of Lemmas 3, 4, and 5 in~\citet{huang2025generalization}.

We can then give the proof of~\cref{thm:reversal_happen}.
\begin{proof}
    By~\eqref{eq:symmetry_solution_no_idn}, $ \forall a\in \cA \backslash \{ r_{-}(b) \}, h_{[b,r_{-}],a}(\bWOV) = g_1$. Thus, it suffices to show that $g_1 = 0$ for the optimal solution.
    
    By~\cref{lemma:symmetry_solution_no_idn}, the nuclear norm of $\| \bWOV \|_{\star}$ is
    \begin{equation}
        \| \bWOV \|_{\star} = \sigma_1^{(1)} + \sigma_1^{(2)} + (N-1) \sigma_2^{(1)} + (N-1) \sigma_2^{(2)}.
    \end{equation}

    $\sigma_1^{(1)} + \sigma_1^{(2)} $ is
    \begin{align}\allowdisplaybreaks
        & \sigma_1^{(1)} +  \sigma_1^{(2)}  \\
        &
        \begin{aligned}
            =& \sqrt{\frac{(C_{A1} + NC_{A2} + C_{D1} + NC_{D2}) + \sqrt{(C_{A1} + NC_{A2} - (C_{D1} + NC_{D2}))^2 + 4(C_{B1} + NC_{B2})^2}}{2}} \\
        &+ \sqrt{\frac{(C_{A1} + NC_{A2} + C_{D1} + NC_{D2}) - \sqrt{(C_{A1} + NC_{A2} - (C_{D1} + NC_{D2}))^2 + 4(C_{B1} + NC_{B2})^2}}{2}}
        \end{aligned}\\
        =&  \sqrt{C_{A1} + N C_{A2} + C_{D1} + N C_{D2} + 2 \sqrt{(C_{A1} + N C_{A2}) (C_{D1} + N C_{D2}) - (C_{B1} + N C_{B2})^2}}
    \end{align}

    $\sigma_2^{(1)} + \sigma_2^{(2)} $ is
    \begin{align}
        \sigma_2^{(1)} + \sigma_2^{(2)}  
        = & \sqrt{\frac{(C_{A1} + C_{D1}) + \sqrt{(C_{A1} - C_{D1})^2 + 4C_{B1}^2}}{2}} +  \sqrt{\frac{(C_{A1} + C_{D1}) - \sqrt{(C_{A1} - C_{D1})^2 + 4C_{B1}^2}}{2}}\\
        = & \sqrt{C_{A1} + C_{D1} + 2 |f_1 m_1 -g_1 l_1|}
    \end{align}

    Therefore, the optimization problem \eqref{eq:SVM} can be reformulated as
    \begin{equation}\label{eq:reversal_happen_proof_problem_1}
            \begin{aligned}
    & \begin{aligned}
                \min  \| \bWOV \|_{\star} =& \sqrt{C_{A1} + N C_{A2} + C_{D1} + N C_{D2} +
                 2 \sqrt{(C_{A1} + N C_{A2}) (C_{D1} + N C_{D2}) - (C_{B1} + N C_{B2})^2}} \\
        & + (N-1) \sqrt{C_{A1} + C_{D1} + 2 |f_1 m_1 -g_1 l_1|}
    \end{aligned}\\
    &  \quad \text{s.t. } \quad l_1 \geq 1, \quad l_1 + l_2 + \alpha \geq f_1 + f_2 + \beta + 1, \quad l_1 + l_2 + \alpha \geq  f_2 + \beta + 1. 
    \end{aligned}
    \end{equation}

    Notice that by the Cauchy-Schwarz inequality, we have
\begin{align}
    &2 \sqrt{(C_{A1} + N C_{A2}) (C_{D1} + N C_{D2}) - (C_{B1} + N C_{B2})^2}\\
    & \geq 0\\
    &\geq -\frac{1}{2} [(f_1+Nf_2+l_1+Nl_2)^2+(g_1+Ng_2+m_1+Nm_2)+N(\beta+\alpha)^2],
\end{align}
where the first equality condition is $|f_1+Nf_2| = |l_1+Nl_2|, |g_1+N g_2|=|m_1+N m_2|$, and $|\beta| = |\alpha|$. 

Therefore,
\begin{align}
    &\sqrt{C_{A1} + N C_{A2} + C_{D1} + N C_{D2} + 2 \sqrt{(C_{A1} + N C_{A2}) (C_{D1} + N C_{D2}) - (C_{B1} + N C_{B2})^2}}\\
    \geq &\sqrt{C_{A1} + N C_{A2} + C_{D1} + N C_{D2} -\frac{1}{2} [(f_1+Nf_2+l_1+Nl_2)^2+(g_1+Ng_2+m_1+Nm_2)+N(\beta+\alpha)^2]}\\
    =& \sqrt{ \frac{(f_1 + Nf_2 - l_1-N l_2)^2 + (m_1 + N m_2 - g_1-N g_2)^2 + N (\beta - \alpha)^2}{2} },\label{eq:reversal_happen_proof_eq1}
\end{align}
where equality holds if and only if $f_1+Nf_2+l_1+Nl_2 = g_1+N g_2+m_1+N m_2 = \beta+\alpha = 0$.

Inspired by~\eqref{eq:reversal_happen_proof_eq1}, we define the following new variables $ c_{f-l}, c_{m-g}$ and $c_{\beta - \alpha}$
\begin{equation}
    c_{f-l} = f_1 + Nf_2 - l_1-N l_2, \quad c_{m-g} = m_1 + N m_2 - g_1 - N g_2, \quad c_{\beta - \alpha} = \beta - \alpha.
\end{equation}

Then, the condition in~\eqref{eq:reversal_happen_proof_problem_1} can be written as
\begin{equation}
    l_1 \geq 1, \quad  (N-1) l_1 \geq N c_{\bmg} + c_{f-l}  +(N-1)f_1  + N, \quad (N-1) l_1 \geq N c_{\bmg} + c_{f-l}  - f_1  + N.
\end{equation}

Combining these observations, we can derive a lower bound problem for problem~\eqref{eq:reversal_happen_proof_problem_1}.
    \begin{equation}\label{eq:reversal_happen_proof_problem_lb}
            \begin{aligned}
                \min  \| \bWOV \|_{\star} &= \sqrt{\frac{c_{f-l}^2 + c_{m-g}^2+N c_{\bmg}^2}{2}}
        + (N-1) \sqrt{ (f_1-m_1)^2 + (g_1-l_1)^2 + 4t } \\
      \text{s.t. } \quad q_1: &  \ l_1 \geq 1,\\
    q_2: & \ (N-1) l_1 \geq (N-1) f_1 + c_{f-l} + N c_{\bmg}  + N,\\
        q_3: & \ (N-1) l_1 \geq - f_1 + c_{f-l} + N c_{\bmg}  + N,\\
    q_4: & \ t\geq f_1m_1,\\
   q_5: & \   t \geq g_1l_1.
    \end{aligned}
    \end{equation}

Notice that the lower bound problem~\eqref{eq:reversal_happen_proof_problem_lb} is optimized over variables $f_1,g_1,l_1,m_1,c_{f-l},c_{m-g}$, $c_{\bmg}$, and $t$. When these parameters are fixed, we can always find a set of other parameters such that the equality condition $f_1+Nf_2+l_1+Nl_2 = g_1+N g_2+m_1+N m_2 = \beta+\alpha = 0$ holds. Therefore, the lower bound problem~\eqref{eq:reversal_happen_proof_problem_lb} has the same optimal solution and optimal value as problem~\eqref{eq:reversal_happen_proof_problem_1}.

We notice that $l_1 = 1, f_1=m_1=g_1 =0, c_{f-l} = c_{\bmg} = -\frac{1}{N+1},   c_{m-g} =0, t = 0$ is a feasible solution with objective value $N-1 + \sqrt{\frac{1}{2(N+1)}}$. It serves as a baseline solution to help us eliminate suboptimal solutions (actually, it is the optimal solution).

Let us denote $M_1 =\frac{c_{f-l}^2 + c_{m-g}^2+N c_{\bmg}^2}{2} $  and $M_2 = (f_1-m_1)^2 + (g_1-l_1)^2 + 4t $ for notation convenience. Then, the KKT condition for problem~\eqref{eq:reversal_happen_proof_problem_lb} is
\begin{align}
    \gamma_i \geq & 0, i =1,2,3,4,5\\
    l_1 :& (N-1) \frac{l_1 - g_1}{\sqrt{M_2}} - \gamma_1 - (N-1) \gamma_2 - (N-1) \gamma_3 + g_1 \gamma_5 = 0\label{eq:reversal_happen_proof_kkt_eq1}\\
    f_1 :& (N-1) \frac{f_1 - m_1}{\sqrt{M_2}} + (N-1) \gamma_2 -\gamma_3 + m_1\gamma_4 = 0\label{eq:reversal_happen_proof_kkt_eq2}\\
    g_1 :& (N-1) \frac{g_1 - l_1}{\sqrt{M_2}} + l_1\gamma_5=0\label{eq:reversal_happen_proof_kkt_eq3}\\
    m_1: & (N-1) \frac{m_1 - f_1}{\sqrt{M_2}} + f_1\gamma_4=0\label{eq:reversal_happen_proof_kkt_eq4}\\
    c_{f-l} : & \frac{c_{f-l}}{2\sqrt{M_1}} + \gamma_2 + \gamma_3  = 0\label{eq:reversal_happen_proof_kkt_eq5}\\
    c_{\bmg} : &  N \frac{c_{\bmg}}{2\sqrt{M_1}} + N \gamma_2 + N\gamma_3 = 0\label{eq:reversal_happen_proof_kkt_eq6}\\
    c_{m-g}: & \frac{c_{m-g}}{2\sqrt{M_1}} =0\label{eq:reversal_happen_proof_kkt_eq7}\\
    t: & (N-1) \frac{2}{\sqrt{M_2}} -\gamma_4 - \gamma_5\label{eq:reversal_happen_proof_kkt_eq8} =0,
\end{align}
where $\gamma_i,i\in [5]$ are the dual variables corresponding to the constraints $q_i,i\in [5]$.

By~\cref{eq:reversal_happen_proof_kkt_eq7}, $c_{m-g} = 0$. By~\cref{eq:reversal_happen_proof_kkt_eq5} and~\cref{eq:reversal_happen_proof_kkt_eq6}, $c_{\bmg} = c_{f-l}$. We then prove that $f_1 = 0$ by case analysis.


\textbf{Case 1: $f_1 > 0$}

In this case, constraint $q_2$ is strictly stronger than $q_3$, so $q_3$ can not be tight, which indicates that $\gamma_3=0$ by complementary slackness.

By computing \eqref{eq:reversal_happen_proof_kkt_eq2} $-$ \eqref{eq:reversal_happen_proof_kkt_eq4}, we have
\begin{align}\label{eq:reversal_happen_proof_eq2}
   0 = \eqref{eq:reversal_happen_proof_kkt_eq2} -\eqref{eq:reversal_happen_proof_kkt_eq4} = (f_1-m_1)(\frac{2(N-1)}{\sqrt{M_2}} - \gamma_4) + (N-1)\gamma_2
\end{align}

By~\cref{eq:reversal_happen_proof_kkt_eq4} and~\cref{eq:reversal_happen_proof_kkt_eq8}, 
\begin{align}
    f_1-m_1 =  f_1\gamma_4 \frac{\sqrt{M_2}}{N-1} \geq 0\\
    \frac{2(N-1)}{\sqrt{M_2}} - \gamma_4 = \gamma_5 \geq 0
\end{align}

Therefore, in~\cref{eq:reversal_happen_proof_eq2}, both $(f_1-m_1)(\frac{2(N-1)}{\sqrt{M_2}} - \gamma_4) $ and $(N-1)\gamma_2$ are non-negative, which means at least one of $f_1-m_1$ and $\frac{2(N-1)}{\sqrt{M_2}} - \gamma_4$ is 0.

If $f_1-m_1 = 0$, then by~\eqref{eq:reversal_happen_proof_kkt_eq4}, $\gamma_4 =0$, so by~\eqref{eq:reversal_happen_proof_kkt_eq8}, $\gamma_5 = \frac{2(N-1)}{\sqrt{M_2}}$. Plugging $\gamma_5= \frac{2(N-1)}{\sqrt{M_2}}$ into~\eqref{eq:reversal_happen_proof_kkt_eq3}, we have $g_1 = -l_1 \leq -1$. Then, the objective value $ \| \bWOV \|_{\star}\geq
         (N-1) \sqrt{ (f_1-m_1)^2 + (g_1-l_1)^2 + 4t } \geq  (N-1) \sqrt{  (g_1-l_1)^2 } \geq 2(N-1) > N-1 +\sqrt{\frac{1}{2(N+1)}} $. Thus, it can not be the optimal solution.

If $\frac{2(N-1)}{\sqrt{M_2}} - \gamma_4 =0$, then by~\eqref{eq:reversal_happen_proof_kkt_eq8}, $\gamma_5 = 0$. According to~\eqref{eq:reversal_happen_proof_kkt_eq3}, $g_1 = l_1 \geq 1$. Then, the objective value $ \| \bWOV \|_{\star}\geq
         (N-1) \sqrt{ (f_1-m_1)^2 + (g_1-l_1)^2 + 4t } \geq  (N-1) \sqrt{  (g_1+l_1)^2 } \geq 2(N-1) > N-1 +\sqrt{\frac{1}{2(N+1)}} $. Thus, it also can not be the optimal solution.

Therefore, there does not exist an optimal solution with $f_1 > 0$.

\textbf{Case 2: $f_1 < 0$}

Similarly, by complementary slackness, $\gamma_2 = 0$ and we have
\begin{equation}
    0 = \eqref{eq:reversal_happen_proof_kkt_eq2}-~\eqref{eq:reversal_happen_proof_kkt_eq4} = (f_1-m_1)(\frac{2(N-1)}{\sqrt{M_2}} - \gamma_4) - \gamma_3
\end{equation}

Similarly, by \cref{eq:reversal_happen_proof_kkt_eq4} and \cref{eq:reversal_happen_proof_kkt_eq8}, $f_1-m_1 \leq 0$, and $(\frac{2(N-1)}{\sqrt{M_2}} - \gamma_4) \geq 0$,  so we have $(f_1-m_1)(\frac{2(N-1)}{\sqrt{M_2}} - \gamma_4) =0$. By the same argument as in Case 1, there does not exist an optimal solution with $f_1 < 0$.

Therefore, the optimal solution must have $f_1=0$, so $f_1 m_1 =0$.

We next show that the optimal $g_1 =0$ by using the tightness of constraints $q_4$ and $q_5$.

If $g_1 > 0$, then $g_1 l_1 > 0 = f_1m_1$, so constraint $q_4$ is not tight. Therefore, $\gamma_4 =0$ and $\gamma_5 = \frac{2(N-1)}{\sqrt{M_2}}$. This indicates $g_1 = - l_1 \leq -1$ by~\eqref{eq:reversal_happen_proof_kkt_eq3}, which contradict with $g_1 > 0$.

If $g_1 < 0$, then constraint $q_5$ is not tight, so $\gamma_5 =0$, so $g_1 = l_1 \geq 1$ by~\eqref{eq:reversal_happen_proof_kkt_eq3}, which contradict with $g_1 < 0$.

Therefore, the optimal $g_1$ is 0. So $h_{[b,r_{-}],a}(\bWOV^{\textup{+}}) = 0, \forall a\in \cA \backslash \{ r_{-}(b) \}$.
\end{proof}

\subsection{Proof of~\cref{thm:idn_bridge}}
\label{sec:proof_idn_bridge}

\begin{proof}
    We first note that~\cref{lemma:symmetry_solution_no_idn} also hold with $\cD = \cD_{r_{+}} \cup \cD_\text{idn} $, so by~\eqref{eq:symmetry_solution_no_idn}, $ \forall a\in \cA \backslash \{ r_{-}(b) \}$, it holds that $ h_{[b,r_{-}],a}(\bWOV) = g_1$ and $ \forall a\in \cB$, we have $h_{[b,r_{-}],a}(\bWOV) = \min \{g_1 + g_2 - \beta - (m_1 + m_2 -\alpha), g_1 + g_2 - \beta - ( m_2 -\alpha) \}$. Since the identity bridge data requires $m_1\geq 1$ (as shown in~\eqref{eq:idn_bridge_proof_problem_1}),
    %
    it suffices to prove the following two inequalities
    \begin{align}
        g_1 >0, \quad  g_1 + g_2 - \beta - (m_1 + m_2 -\alpha)>0.
    \end{align}

    According to~\cref{lemma:symmetry_solution_no_idn}, we can rewrite the problem~\cref{eq:SVM} as
    \begin{equation}\label{eq:idn_bridge_proof_problem_1}
            \begin{aligned}
    & \begin{aligned}
                & \min  \| \bWOV \|_{\star} \\  =& \sqrt{C_{A1} + N C_{A2} + C_{D1} + N C_{D2} +
                 2 \sqrt{(C_{A1} + N C_{A2}) (C_{D1} + N C_{D2}) - (C_{B1} + N C_{B2})^2}} \\
        & + (N-1) \sqrt{C_{A1} + C_{D1} + 2 |f_1 m_1 -g_1 l_1|},
    \end{aligned}\\
    &  \quad \text{s.t. } \ \  f_1\geq 1, l_1 \geq 1, m_1 \geq 1,\\
    & \quad \quad \quad l_1 + l_2 + \alpha \geq f_1 + f_2 + \beta + 1,\\
    & \quad \quad \quad f_1 + f_2 \geq l_1 + l_2 +1,\\
    & \quad \quad \quad m_1 + m_2 \geq g_1 + g_2 +1,\\
    &\quad \quad \quad  m_1+m_2 \geq g_2 +1.
    \end{aligned}
    \end{equation}

    Similar to the proof of~\cref{thm:reversal_happen}, we define $ c_{f-l} = f_1 + Nf_2 - l_1-N l_2, c_{m-g} = m_1 + N m_2 - g_1 - N g_2  , c_{\beta - \alpha} = \beta - \alpha$ and have the following lower bound problem
        \begin{equation}\label{eq:idn_bridge_proof_problem_lb}
            \begin{aligned}
                \min  \| \bWOV \|_{\star} &= \sqrt{\frac{c_{f-l}^2 + c_{m-g}^2+N c_{\bmg}^2}{2}}
        + (N-1) \sqrt{ (f_1-m_1)^2 + (g_1-l_1)^2 + 4t } \\
      \text{s.t. } \quad q_1:&   f_1 \geq 1,\\
      q_2:&   l_1 \geq 1,\\
      q_3:&   m_1 \geq 1,\\
    q_4:& (N-1) l_1 \geq (N-1) f_1 + c_{f-l} + N c_{\bmg}  + N,\\
    q_5:& (N-1) f_1 + c_{f-l}  \geq (N-1)l_1 + N,\\
    q_6:& (N-1) m_1 + c_{m-g} \geq (N-1) g_1 + N,\\
    q_7: & (N-1) m_1 + c_{m-g} \geq - g_1 + N,\\
    q_8:&  t\geq f_1m_1,\\
   q_9:&   t \geq g_1l_1.
    \end{aligned}
    \end{equation}

The lower bound problem~\eqref{eq:idn_bridge_proof_problem_lb} is optimized over variables $f_1,g_1,l_1,m_1,c_{f-l},c_{m-g},c_{\bmg},t$. When these parameters are fixed, we can always find a set of other parameters such that the equality condition $f_1+Nf_2+l_1+Nl_2 = g_1+N g_2+m_1+N m_2 = \beta+\alpha = 0$ holds. Therefore, the lower bound problem~\eqref{eq:idn_bridge_proof_problem_lb} has the same optimal solution and optimal value as problem~\eqref{eq:idn_bridge_proof_problem_1}.

Notice that
\begin{equation}\label{eq:idn_bridge_proof_feasible_sol}
    f_1=m_1=l_1=g_1 = 1,c_{f-l}= c_{m-g} = N,c_{\bmg}=-2, t = 1
\end{equation}
is a feasible solution with objective value
\begin{equation}
    \sqrt{\frac{2N^2 + 4N}{2}} + 2(N-1) \leq (N+1) + 2(N-1) = 3N-1,
\end{equation}
which means any optimal solution must have an objective value no larger than $3N-1$.

Let $M_1 =\frac{c_{f-l}^2 + c_{m-g}^2+N c_{\bmg}^2}{2} $  and $M_2 = (f_1-m_1)^2 + (g_1-l_1)^2 + 4t $. The KKT condition of problem~\eqref{eq:idn_bridge_proof_problem_lb} is
\begin{align}
    \gamma_i \geq & 0, i \in [9]\\
    l_1 :& (N-1) \frac{l_1 - g_1}{\sqrt{M_2}} - \gamma_2 - (N-1) \gamma_4 +(N-1)\gamma_5  + g_1 \gamma_9 = 0\label{eq:idn_bridge_proof_kkt_eq1}\\
    f_1 :& (N-1) \frac{f_1 - m_1}{\sqrt{M_2}} - \gamma_1 + (N-1) \gamma_4 -(N-1)\gamma_5 + m_1\gamma_8 = 0\label{eq:idn_bridge_proof_kkt_eq2}\\
    g_1 :& (N-1) \frac{g_1 - l_1}{\sqrt{M_2}} + (N-1) \gamma_6 - \gamma_7 + l_1\gamma_9=0\label{eq:idn_bridge_proof_kkt_eq3}\\
    m_1: & (N-1) \frac{m_1 - f_1}{\sqrt{M_2}} - \gamma_3 -(N-1)\gamma_6 - (N-1)\gamma_7  + f_1\gamma_8=0\label{eq:idn_bridge_proof_kkt_eq4}\\
    c_{f-l} : & \frac{c_{f-l}}{2\sqrt{M_1}} + \gamma_4 - \gamma_5  = 0\label{eq:idn_bridge_proof_kkt_eq5}\\
    c_{m-g}: & \frac{c_{m-g}}{2\sqrt{M_1}} - \gamma_6 -\gamma_7 =0\label{eq:idn_bridge_proof_kkt_eq6}\\
    c_{\bmg} : &  N \frac{c_{\bmg}}{2\sqrt{M_1}} + N \gamma_4 = 0\label{eq:idn_bridge_proof_kkt_eq7}\\
    t: & (N-1) \frac{2}{\sqrt{M_2}} -\gamma_8 - \gamma_9\label{eq:idn_bridge_proof_kkt_eq8} =0,
\end{align}
where $\gamma_i,i\in [9]$ are the dual variables corresponding to the constraints $q_i,i\in [9]$.

First, from constraints $q_4$ and $q_5$, we have
\begin{align}
    (N-1)l_1 \geq (N-1) f_1 + c_{f-l} + N c_{\bmg}  + N \geq (N-1)l_1 + N + N c_{\bmg} + N 
    \Rightarrow c_{\bmg}\leq -2,
\end{align}
so by~\eqref{eq:idn_bridge_proof_kkt_eq7}, $\gamma_4 > 0$, which indicates constraint $q_4$ is tight.

We then prove the two inequalities.

\textbf{Part 1: $g_1 > 0$}

We prove by contradiction. Suppose $g_1 \leq 0$. Then, constraint $q_9$ is not tight, so $\gamma_9 = 0$. By~\eqref{eq:idn_bridge_proof_kkt_eq3}
\begin{equation}\label{eq:idn_bridge_proof_eq2}
    \gamma_6 = \frac{1}{N-1} \gamma_7 +   \frac{l_1 - g_1}{\sqrt{M_2}} > 0,
\end{equation}
which means constraint $q_6$ is tight.

Plugging~\eqref{eq:idn_bridge_proof_eq2} into~\eqref{eq:idn_bridge_proof_kkt_eq6}, we have
\begin{equation}
    \sqrt{M_2} = \frac{l_1-g_1}{\frac{c_{m-g}}{2\sqrt{M_1}}-\frac{N}{N-1}\gamma_7} \geq \frac{l_1-g_1}{\frac{c_{m-g}}{2\sqrt{M_1}}} \geq \frac{2\sqrt{M_1}}{c_{m-g}}. 
\end{equation}

By constraint $q_6$,
\begin{equation}
    c_{m-g} = (N-1)g_1 + N - (N-1)m_1 \leq 1.
\end{equation}

Therefore,
\begin{equation}
    \sqrt{M_2}  \geq  \frac{2\sqrt{M_1}}{c_{m-g}}  \geq 2\sqrt{M_1}.
\end{equation}

So the objective is lower bounded by
\begin{equation*}
\begin{aligned}
    \| \bWOV \|_{\star} =& \sqrt{M_1} + (N-1)\sqrt{M_2} \geq \sqrt{M_1} (2N-1) \\ \geq & \sqrt{\frac{N c_{\bmg}^2}{2}}(2N-1) \geq \sqrt{2N} (2N-1) \geq 4N-2 > 3N-1,
\end{aligned}
\end{equation*}
and thus it is suboptimal to the known feasible solution~\eqref{eq:idn_bridge_proof_feasible_sol}.

Therefore, we must have $g_1 > 0$

\textbf{Part 2: $g_1 + g_2 - \beta - (m_1 + m_2 -\alpha)>0$}

This inequality is equivalent to
\begin{equation}
    (N-1)g_1 - (N-1)m_1 - c_{m-g} - N c_{\bmg} > 0
\end{equation}

To prove this inequality, we discuss two cases.

Case 1: constraint $q_6:(N-1) m_1 + c_{m-g} \geq (N-1) g_1 + N$ is tight. Then,
\begin{equation}
     (N-1) g_1 - (N-1) m_1 - c_{m-g}  = - N,
\end{equation}
so 
\begin{equation}
     (N-1)g_1 - (N-1)m_1 - c_{m-g} - N c_{\bmg} = -N - N c_{\bmg} \geq -N - N\cdot(-2) = N.
\end{equation}

Case 2: constraint $q_6$ is not tight. Then, $\gamma_6=0$. From part 1, we know $g_1 > 0$, so $q_7$ can not be tight (otherwise $q_6$ can not hold), which means $\gamma_7 = 0$. Combining these with~\eqref{eq:idn_bridge_proof_kkt_eq6}, we conclude that $c_{m-g} = 0$. Therefore,
\begin{equation}\label{eq:idn_bridge_proof_eq3}
    (N-1)g_1 - (N-1)m_1 - c_{m-g} - N c_{\bmg}  = (N-1)g_1 - (N-1)m_1  - N c_{\bmg} >  - (N-1)m_1 + 2 N,
\end{equation}
where last inequality uses $g_1 > 0$ and$ c_{\bmg}\leq 2$.

By~\eqref{eq:idn_bridge_proof_eq3}, it suffices to show $m_1\leq 2$. Supposing the opposite, the objective is lower bounded by
\begin{equation}
     \| \bWOV \|_{\star} = \sqrt{M_1} + (N-1)\sqrt{M_2} \geq 2 + (N-1) \sqrt{(m_1+f_1)^2} > 2 + (N-1) (2+1) = 3N-1,
\end{equation}
which is suboptimal to the known feasible solution~\eqref{eq:idn_bridge_proof_feasible_sol}.

Therefore, in both cases, we have $(N-1)g_1 - (N-1)m_1 - c_{m-g} - N c_{\bmg} > 0$.
\end{proof}

%% file: appendix/appendix_experiments.tex
\section{Experiment Details}
\label{app:experiment_detail}
We provide implementation details on experiments in \Cref{sec:exp_one_layer_tf} and \Cref{sec:exp_real_llm}. In the one-layer transformer experiments, we use $N=10$ reversal pairs, and use uniform initialization for $\bW_O$ and $\bW_V$ over $[-0.1,0.1]$. In training, we use full-batch gradient descent with a learning rate of $0.001$ to optimize the model.

For real LLM experiments, we use the cross-entropy loss and optimize it with the AdamW optimizer~\citep{kingma2014adam}. In finetuning, we disable the dataset shuffling in training to make sure the data of the same instant are trained in the same batch, and the different random seeds run are conducted by shuffling the dataset before training. In experiments, we search the best learning rate among $\{ 5\cdot10^{-5}, 10^{-4}, 2\cdot 10^{-4}, 4\cdot 10^{-4}, 7\cdot 10^{-4}, 10^{-3} \}$ and the best weight decay among $\{0.25,0.30,0.35\}$. The batch size is set to $1/10$ of the total data size.

The training loss of the FWD and OCR-6 methods is presented~\cref{fig:train_loss_fwd_ocr}. The training loss of the IDN and OCR-$k$ methods is presented~\cref{fig:train_loss_idn_ocr}. The training loss of the IDN and OCR-$k$ methods is presented~\cref{fig:train_loss_name_len}.

\begin{figure}[t]
  \begin{center}
    \centerline{\includegraphics[width=  \columnwidth]{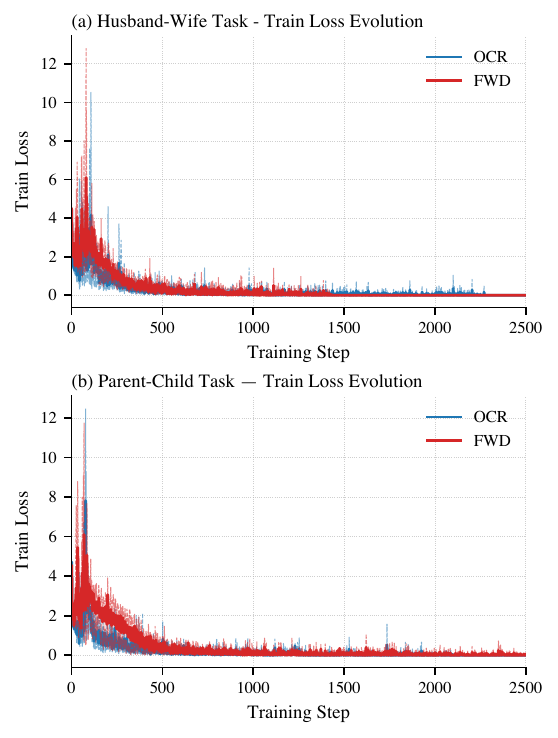}}
    \caption{
      Training loss comparison of FWD and OCR.
    }
    \label{fig:train_loss_fwd_ocr}
  \end{center}
  \vspace{-20pt}
\end{figure}

\begin{figure}[b]
  \begin{center}
    \centerline{\includegraphics[width=  \columnwidth]{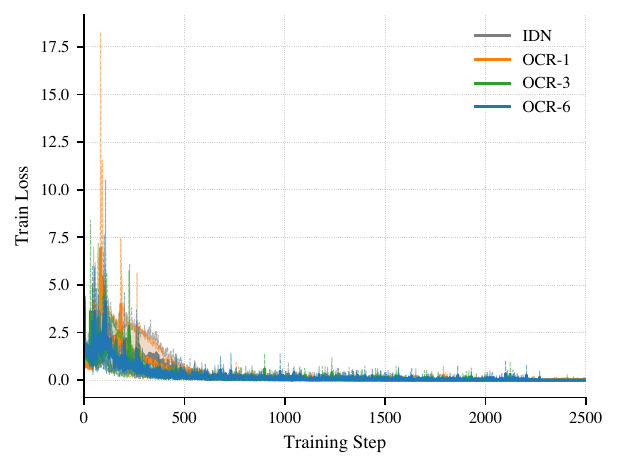}}
    \caption{
      Training loss comparison of IDN and OCR-$k$ on Husband-Wife task.
    }
    \label{fig:train_loss_idn_ocr}
  \end{center}
  \vspace{-20pt}
\end{figure}

\begin{figure}[bph!]
  \begin{center}
    \centerline{\includegraphics[width=  \columnwidth]{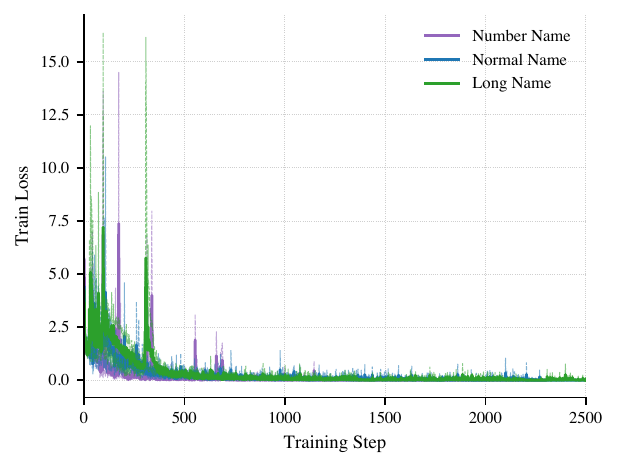}}
    \caption{
      Training loss comparison of different name length on Husband-Wife task.
    }
    \label{fig:train_loss_name_len}
  \end{center}
\end{figure}